\documentclass{article}
\usepackage{kr}

\usepackage{times}
\usepackage{soul}
\usepackage{url}
\usepackage[hidelinks]{hyperref}
\usepackage[utf8]{inputenc}
\usepackage[small]{caption}
\usepackage{graphicx}
\usepackage{amsmath}
\usepackage{amsthm}
\usepackage{booktabs}
\usepackage[linesnumbered,ruled,vlined] {algorithm2e}
\newtheorem{theorem}{Theorem}

\title{Discovering User-Interpretable Capabilities of Black-Box Planning Agents}
\author{%
    Pulkit Verma{\normalfont,} 
    Shashank Rao Marpally{\normalfont, and} 
    Siddharth Srivastava
    \affiliations
    Autonomous Agents and Intelligent Robots Lab,\\
    School of Computing and Augmented Intelligence,
    Arizona State University, USA
    \emails
    \{verma.pulkit, smarpall, siddharths\}@asu.edu    
}

\usepackage{graphicx}
\usepackage{amssymb,amsmath,amsthm}
\usepackage{pict2e}
\usepackage{subcaption}
\usepackage{multirow}
\usepackage{tabularx}
\usepackage{makecell}
\usepackage{fancyvrb}
\usepackage{multirow}
\usepackage{booktabs}
\usepackage[svgnames,table]{xcolor}
\usepackage[most]{tcolorbox}

\usepackage{transparent}

\theoremstyle{definition}

\newtheorem{definition}{Definition}

\newcommand{\ti}[1]{\tilde{#1}}
\newcommand{\alg}{Alg.~\ref{alg:ia}}

\newcommand{\mysssection}[1]{\noindent\textbf{#1}\hspace{8pt}}
  
\usepackage{tikz}
\usetikzlibrary{decorations.pathreplacing,calc}

\newcommand{\citet}[1]{\citeauthor{#1}~(\citeyear{#1})}

\newcommand{\tuple}[1]{\langle#1 \rangle}

\begin{document}

\maketitle

\begin{abstract}
Several approaches have been developed for answering users' specific
questions about AI behavior and for assessing their core functionality
in terms of primitive executable actions. However, the problem of
summarizing an AI agent's broad capabilities for a user is comparatively new.
This paper presents an algorithm for discovering from scratch the
suite of high-level ``capabilities'' that an AI system with arbitrary
internal planning algorithms/policies can perform. It computes
conditions describing the applicability and effects of these
capabilities in user-interpretable terms. Starting from a set of
user-interpretable state properties, an AI agent, and a
simulator that the agent can interact with, our algorithm returns
a set of high-level capabilities with their
parameterized descriptions. Empirical
evaluation on several game-based scenarios shows that this approach
efficiently learns descriptions of various types of AI
agents in deterministic, fully observable settings. User studies show
that such descriptions are easier to understand and
reason with than the agent's primitive actions.
\end{abstract}

\section{Introduction}
\label{sec:introduction}

AI systems are rapidly developing to an extent where their
users may not understand what they can and cannot do safely.  
In fact, the limits and capabilities of many AI
systems are not always immediately clear even to the experts,
as they may use black box policies, e.g., ATARI game-playing 
agents~\cite{pgreydanus_2018_visualizing}, text summarization
tools~\cite{paulus2018a}, mobile
manipulators~\cite{popov2017data}, etc.

Ongoing research on the topic
focuses on the significant problem of how to answer users'
questions about the system's
behavior
while assuming that the user and AI share a common
action vocabulary~\cite{chakraborti2017plan,dhurandhar_18_explanations,Anjomshoae_2019_explainable,Barredo_2020_explainable}. Furthermore, most non-experts hesitate to
ask questions about new AI tools~\cite{Mou_2017_media} and often
do not know which questions to ask for assessing the safe limits
and capabilities of an AI system. This problem is aggravated in
situations where an AI system can carry out planning or sequential
decision making.
Lack of understanding about the limits of an imperfect
system can result in unproductive usage or, in the worst-case, serious
accidents~\cite{Randazzo_18_uber}. This, in turn, limits the adoption
and productivity of AI systems.

This work presents a new approach for 
\emph{discovering} capabilities of a black-box AI system.
The AI system
may use arbitrary internal models, representations, and processes for
computing solutions to user-assigned tasks. 
Prior work on the topic addresses complementary problems
of deriving symbolic descriptions for pre-defined
skills~\cite{konidaris_18_from} and of learning users' conceptual
vocabularies~\cite{kim_2018_interpretability,sreedharan_2022_bridging}.
However, they do not address the problem of
\emph{discovering} high-level user-interpretable capabilities that an agent
can perform using arbitrary, internal behavior synthesis algorithms
(see Sec.\!~\ref{sec:related_work} for a greater discussion).

\begin{figure*}
    \centering
    \includegraphics[width=\textwidth]{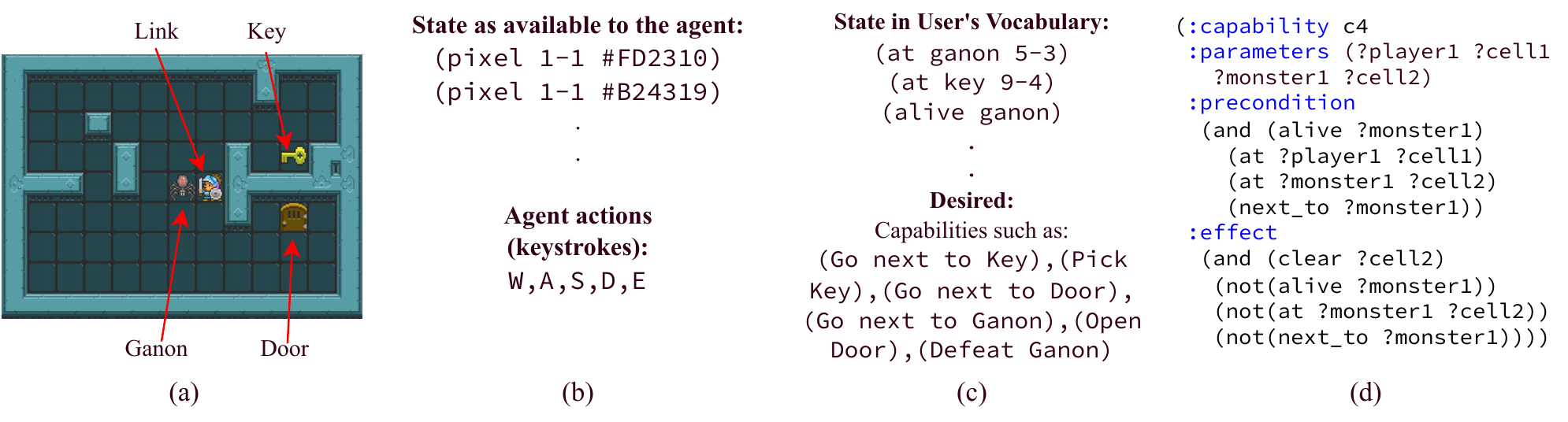}
    \caption{From pixels to interpretable capabilities. (a) A Zelda-like game; (b) States available to the agent and its actions; (c) States represented in user vocabulary, and possible set of desired capabilities; (d) A parameterized capability description learned by our method.}
    \label{fig:abstraction}
\end{figure*}

As a starting point, in this paper, we assume determinism and full
observability on part of the AI system. 
Since there are no solution approaches for solving the
problem even in this foundational setting, our framework
can serve as a foundation for solutions to the more general
setting in future. 

\subsubsection{Running example} Consider a game based
on ``The Legend of Zelda'' (Fig.\!~\ref{fig:abstraction})
featuring a protagonist player \emph{Link} who must defeat the
antagonist monster \emph{Ganon}, and escape through the door using a key. 
(Fig.\!~\ref{fig:abstraction})(a)
shows the game state as the agent sees it; its primitive actions are
keystrokes as shown in (b). These keystrokes do not help convey the
agent's capabilities because
(i) they are too fine-grained, and (ii) they show the set of
actions available to the AI system, although its true capabilities depend
on its AI planning and learning algorithms.
Fig.\,\ref{fig:abstraction}(c) shows common English terms that a
user might understand (called \emph{user's vocabulary}), and the types of capabilities that they may
want to know about. Fig.\,\ref{fig:abstraction}(d) shows a parameterized
capability discovered by our method. Intuitively, Fig.\,\ref{fig:abstraction}(d) captures
the ``defeat Ganon'' capability.

This paper shows how we can discover and describe
an agent's capabilities in the form of
Fig.\,\ref{fig:abstraction}(d). 
This capability description 
can be readily transcribed as ``If the \emph{player} is at \emph{cell1}; the \emph{monster} is at \emph{cell2}; 
the \emph{monster} is \emph{alive} (not defeated); and 
the \emph{monster} is next to the \emph{player}; 
then the \emph{player} can act to
reach a state where 
\emph{cell2} is empty;
the \emph{monster} is \emph{not alive} (defeated); 
the \emph{monster} is \emph{not at cell2}; and 
the \emph{player} is not next to the \emph{monster}.''
Our empirical evaluation shows that our system effectively discovers
such high-level capabilities; our user study shows that the discovered
capabilities help users effectively estimate black-box agent capabilities.

The rest of this paper is organized as follows. The next section
presents a formal framework for capabilities as well as notions of
correctness for discovered agent capabilities. Sec.\,\ref{sec:approach}
describes our main algorithms and their formal properties and
Sec.\,\ref{sec:experiments} presents empirical results and results from
user studies. Sec.\,\ref{sec:related_work} discusses the
relationship of the presented methods with prior work.
Finally, Sec.\,\ref{sec:conclusion} presents our conclusion and
future directions.

\section{Formal Framework}
\label{sec:background}

We model an AI system (``agent'' henceforth) as a 3-tuple $\langle
S, A, T \rangle$, where $S$ is the state space, $A$ is the set of
actions that the agent can execute, $T: S \times A \rightarrow S$ is a
deterministic black-box transition function determining the effects of
the agent's primitive actions on the environment. For brevity of notation,
we use $a(s)$ to represent $T(s,a)$, where $a \in A$, and $s \in
S$. Given a goal set $G\subseteq S$, a black-box \emph{deterministic
policy} $\Pi: S \rightarrow A$ maps each state to the action that the
agent should execute in that state to reach a $g\in G$.

In this paper, we use ``actions'' to refer to the
core \emph{functionality} of the agent, denoting the 
agent's decision choices, or primitive actions
that the agent could execute (e.g., keystrokes in our running
example). In contrast, we use the term ``capabilities'' to refer to
the \emph{high-level behaviors} that the agent can perform using its AI
algorithms for behavior synthesis, including planning and learning
(e.g., defeating Ganon or picking up the key). 
Thus, actions refer to
the set of choices that a tabular-rasa agent may possess, while
capabilities are a result of its agent
function~\cite{russell1997rationality} and can change as a result of
algorithmic updates even as the agent uses the same actions.

\subsection{Abstraction} 
\label{sec:abstraction}
We now define the notion of abstraction used in this work.
Several approaches have explored the use of abstraction in
planning~\cite{Sacerdoti_74_planning,Giunchiglia_92_theory,Helmert_07_flexible,Backstrom_13_bridging,Srivastava_16_metaphysics}.
We refer to $\ti{S}$ as the set of \emph{high-level} or
\emph{abstract} states, and $S$ as the set of
\emph{low-level} or \emph{concrete} states. We define
abstraction as in \cite{Srivastava_16_metaphysics}:

\begin{definition}
Let $S$ and $\ti{S}$ be sets such that $|\ti{S}| \leq |S|$.
An \emph{abstraction} from $S$ to $\ti{S}$ is defined by a
surjective function $f: S \rightarrow \ti{S}$. For any
$\ti{s} \in \ti{S}$, the concretization function 
$f^{-1} (\ti{s}) = \{s \in S : f(s) = \ti{s}\}$ denotes the
set of states represented by the abstract state $\ti{s}$.
\end{definition}

Following this, we use $\,\,\ti{}\,\,$ whenever we refer to a
state, a predicate, or an action pertaining to the abstract state
space. 

\subsection{Capability Descriptions}
We express
capability descriptions using a STRIPS-like
representation~\cite{Fikes1971,McDermott_1998_PDDL}.
This is because, when used with a user's vocabulary, such
a representation can be readily transcribed into statements such as ``in
situations where $X$ holds, if the agent executes actions
$a_1,\dots,a_k$ it would result in $Y$'', where $X$ and $Y$ are
in the user's vocabulary~\cite{Camacho19,verma_21_asking}.
Such representations have been shown to be intuitive for humans in 
understanding deliberative behaviors of other
agents~\cite{malle2004mind,miller_2019_explanation}.
In our running example, such a description
could indicate that if Link
is next to Ganon then Link can defeat it. 
We now formally define
a capability.

\begin{definition}
\label{def:cd}
  Given a set of objects $\ti{O}$; and a finite set
  of predicates $\ti{P} = \{\ti{p}_1^{k_1}, \ldots, \ti{p}_n^{k_n}\}$
  with arities $k_i$;
  a \emph{grounded capability} $\ti{c}^*$ is defined as a tuple
  $\langle \emph{pre}(\ti{c}^*), \emph{eff}(\ti{c}^*)\rangle$
  where precondition 
  $\emph{pre}(\ti{c}^*)$ and effect $\emph{eff}(\ti{c}^*)$ are 
  conjunctions of literals over $\ti{P}$ and $\ti{O}$.
\end{definition}

We also refer to the tuple 
$\langle \ti{c}^*, \emph{pre}(\ti{c}^*), \emph{eff}(\ti{c}^*)\rangle$
as the \emph{capability description} for a capability $\ti{c}^*$.
Here each atom could be absent, positive, or negative (henceforth
referred to as the \emph{mode}) in the
precondition and the effect of an action. However, we disallow atoms to
be positive (or negative) in both the preconditions and the effects
of an action simultaneously to avoid redundancy.
Semantics of 
capabilities are close
to those of STRIPS actions, but they address vocabulary
disparity: an agent can execute a capability $\ti{c}^*$ in any
concrete state $s$ where $\ti{s}\models pre(\ti{c}^*)$; as a
result, the system reaches a concrete state $s'$ (a
member of an abstract state $\ti{s}'$). Atoms
that don't appear in $\emph{eff}(\ti{c}^*)$ retain their truth
values from $\ti{s}$ in $\ti{s}'$ while others are set to their 
modes (positive, negative, or absent) in
$\emph{eff}(\ti{c}^*)$, i.e., $\forall \ell\in
\emph{eff}(\ti{c}^*)$, $\ti{s}'\models \ell$. For brevity, we
represent this as $\ti{s}' = \ti{c}^*(\ti{s})$. 

We refer to the capabilities defined in Def.~\ref{def:cd} as  grounded
capabilities as they are instantiated with a specific set of objects in $\ti{O}$.
We use 
$\,\,^*\,\,$ whenever we refer to a
grounded capability. 
We define a lifted form of capabilities as parameterized capabilities.
\begin{definition}
  Given a set of objects $\ti{O}$; a finite set
  of predicates $\ti{P} = \{\ti{p}_1^{k_1}, \ldots, \ti{p}_n^{k_n}\}$ 
  with arities $k_i$;
  a \emph{parameterized capability} $\ti{c}$ is defined a 3-tuple
  $\langle 
  \emph{args}(\ti{c}),\emph{pre}(\ti{c}), \emph{eff}(\ti{c})\rangle$
  where 
  $\emph{args}(\ti{c})$ is the set of arguments that can be initialized with 
  a set of objects $\ti{o} \subseteq \ti{O}$;
  and $\emph{pre}(\ti{c})$ and $\emph{eff}(\ti{c})$ are sets of
  literals over $\ti{P}$ and $\emph{args}(\ti{c})$.
\end{definition}

A set of parameterized capabilities constitutes a 
parameterized capability model. 
Formally, a \emph{parameterized capability model} is 
a tuple $\ti{M} = \tuple{\ti{P},\ti{C}}$, where $\ti{P}$ is 
a finite set of predicates, and $\ti{C}$ is a finite set of
parameterized capabilities.

Our objective is to develop a capability discovery algorithm that learns
a parameterized capability model of a black-box AI agent using as input (i)
the agent, (ii) a compatible simulator using which the agent can
simulate its primitive action sequences, and (iii) the user's concept
vocabulary, which may be insufficient to express the simulator's state
representation. Such assumptions on the agent are common. 
In fact, the use of third-party simulators for development and testing is the 
bedrock of most of the research on taskable AI systems today (including 
game-playing AI, autonomous cars, and factory robots). Providing 
simulator access for assessment is reasonable as it would allow AI 
developers to retain freedom and proprietary controls on internal 
software while supporting calls for assessment and regulation using 
approaches such as ours.

Our user studies show the efficacy of this approach using spoken English 
terms for concepts without an explicit process for vocabulary synchronization.
Several threads of ongoing research address the problem of identifying
user-specific concept 
vocabularies (e.g., ~\citet{kim_2018_interpretability}, 
\citet{sreedharan_2022_bridging}), 
and the field of intelligent tutoring systems develops methods for helping
users understand a fixed concept
vocabulary. These methods can be used to either elicit or impart a
vocabulary for a given user and such systems can be used to
complement the methods developed in this paper.

However, since the problem of capability discovery is not
well understood even in settings where user-concept
definitions are readily available, we focus on capability
discovery with a given vocabulary with known definitions and
formalize our approach using them. 
Furthermore, our
empirical evaluation and user studies don't place
requirements on user concept vocabularies and show 
the efficacy of this representation. 
We formalize these concept definitions as follows:

\begin{definition}
\label{def:abstraction}
Given a concrete state $s\in S$, 
a set of objects $\ti{O}$ and their tuples $\ti{O}^{\leq d}$
(of dimension at most $d$, where $d$ is a positive integer),
a set of \emph{concepts/predicates} $\ti{P}={\ti{p}_1^{k_1}, \ldots,
\ti{p}_n^{k_n}}$ with their arities $k_i$ and an
associated Boolean evaluation function 
$\psi_{\ti{p}_i}: S\times \ti{P}\times
\ti{O}^{\leq \emph{max}(k_i)} \rightarrow \{T, F\}$, 
$j \leq \emph{max}(k_i)$ , we define
$s\models_{\psi_{\ti{p}_i}} \ti{p_i}(\ti{o}_1, \ldots, \ti{o}_j)$ iff
$\psi_{\ti{p}_i}(s, \ti{p_i}, \ti{o}_1, \ldots, \ti{o}_j)=T$. We define 
\emph{the abstraction $\ti{s}_{\ti{P},\ti{O}}$ of a state} $s\in S$ as 
the set of all literals over $\ti{P}$ and $\ti{O}$ that are true in
$s$. $\ti{S}_{\ti{P}, \ti{O}}$ denotes the abstract state space
$\{\ti{s}_{\ti{P},\ti{O}}: s\in S\}$.
\end{definition}

We omit subscripts $\ti{P}$ and $\ti{O}$ unless needed for
clarity. As mentioned, we assume
availability of an evaluation function $\psi_{\ti{p}}$
associated with each predicate $\ti{p} \in \ti{P}$.
E.g., for a 3-D
Blocksworld simulator with objects $a$ and $b$, and coordinates $x,y,$
and $z$, ``$\emph{on}(a,b)$ is true exactly for states where $z(a)>z(b)$, $x(a)=x(b)$, and
$y(a)=y(b)$.''  As this example illustrates, such vocabularies can be
inaccurate.
The abstraction function $f$ (Def.~\ref{def:abstraction}) 
can be modeled as a conjunction
of these evaluation functions $\psi_{\ti{p}}$.
We now discuss how we discover capabilities and learn their
descriptions. 

\section{Active Capability Discovery}
\label{sec:approach}

\begin{algorithm}[t]
    \SetAlgoNlRelativeSize{-1}
    \DontPrintSemicolon
    \SetKwInOut{Input}{Input $\,$}
    \SetKwInOut{Output}{Output$\,$}
    \caption{Capability Discovery Algorithm}
    \label{alg:ia}
    \Input{predicates $\ti{P}$, agent $\mathcal{A}$}
    \Output{$\ti{M}$}
    $E \leftarrow $ generate\_execution\_traces$(\mathcal{A})$ \;
    $\ti{C}^* \leftarrow $ generate\_partial\_capability\_descriptions$(E)$\;
    $\ti{C}' \leftarrow$ parameterize\_partial\_capabilities$(\ti{C}^*)$\;
    $\ti{M} \gets$ generate\_parameterized\_capability\_model$(\ti{C}')$\;
    Set $\ti{L} \leftarrow \{pre, \emph{eff}\}$\;
    \For{\emph{each} $\langle \ti{L},\ti{C},\ti{P}\rangle$ \emph{in} $\ti{M}$}{
        Generate $\ti{M}_+, \ti{M}_-, \ti{M}_\emptyset$ by setting $\ti{P}$ in $\ti{C}$ at $\ti{L}$ to $+,-,\emptyset$ in $\ti{M}$\;
        \For{each pair $\ti{M}_x, \ti{M}_y$ in $\{\ti{M}_+, \ti{M}_-, \ti{M}_\emptyset\}$}{
            $\ti{q} \leftarrow$ generate\_query$(\ti{M}_1, \ti{M}_2)$\;
            $\ti{\varrho} \leftarrow$ generate\_waypoints$(\ti{q})\,\,\,$ \;
            $\varrho \leftarrow$ refine\_waypoints$(\ti{\varrho}, \ti{P})$ \;  
            \For{$i$ \emph{in range}$[0, k-1]$}{
                $\theta \leftarrow$ ask\_agent$(\mathcal{A}, \langle s_{i}, s_{i+1} \rangle)$\;
                break if $\theta=\bot$\;
            }
            $\ti{M} \leftarrow$ consistent\_description$(i, \ti{s}_i, \ti{M}_x, \ti{M}_y)$\;
        }
    }
    \Return $\ti{M}$\;
\end{algorithm}

Our overall approach consists of two main phases:\linebreak
(1) discovering candidate capabilities and their partial
descriptions from a set of
execution traces of the agent's behavior (Sec.~\ref{sec:discover}); and 
(2) completing the descriptions of the candidate capabilities
discovered in step (1)
by asking the agent queries designed to assess under which conditions
it can execute those capabilities and what their effects are (Sec.~\ref{sec:learn}).
The capability discovery algorithm (\alg) performs both these
steps using user interpretable
predicates $\ti{P}$ and the agent $\mathcal{A}$ as inputs. We now explain
these two phases in detail.

\subsection{Discovering Candidate Partial Capabilities}
\label{sec:discover}

\subsubsection{Generating execution traces} As a first step, \alg\, collects a set of execution traces
$E$ from the agent (line 1). An \emph{execution trace} $e$ is a sequence of states
of the form $\langle s_0, s_1, \dots, s_{n-1}, s_n \rangle$, such that 
$\forall j\in [1,n]\quad \exists a_i \in A$  $a_j(s_{j-1}) = s_j$.
To obtain the traces $e \in E$, a set of random tasks of the 
form $\langle s_I, s_G\rangle$, where $s_I,s_G \in S$,
are given to the agent $\mathcal{A}$, and the agent is asked to 
reach $s_G$ from $s_I$.
Intermediate states that the agent goes through
form the set of execution traces $E$.

\subsubsection{Discovering candidate capabilities} 
To discover candidate capabilities, we abstract the low-level execution traces $E$ 
in terms of the user's vocabulary (line 2).
This abstraction of a low-level execution trace $\langle s_0, s_1, \dots,
s_{n-1}, s_n \rangle$ gives a high-level execution
trace $\langle \ti{s}_0, \ti{s}_1, \dots, \ti{s}_{n-1}, \ti{s}_n \rangle$.
Since we do not assume that the user's vocabulary is precise enough to
discern all the states available to the agent, more than one low-level
state in an execution trace may be abstracted to a single high-level
abstract state in $\ti{S}$. Hence for some $j \in [0,n-1]$, it is possible that
$\ti{s}_j = \ti{s}_{j+1}$. 
E.g., in Fig.\!~\ref{fig:abstraction}(a),
the state available to the agent in the simulator
expresses pixel-level details of the
game (Fig.\!~\ref{fig:abstraction}(b)), whereas the user's vocabulary
can express it only as an abstract state that represents multiple
similar low-level states (Fig.\!~\ref{fig:abstraction}(c)).
Formally, an \emph{abstract execution trace} is the longest subsequence 
of $\ti{s}_1,\dots,\ti{s}_n$ such that no two subsequent elements are identical.
We remove the repetitions from the high-level execution trace
to get the abstract execution trace $\ti{e} = \langle \ti{s}_0,\dots,\ti{s}_m\rangle$,
where $m\leq n$.

We store each transition $\ti{s}_i,\ti{s}_{i+1}$ in $\ti{e}$ as a new 
grounded candidate capability $\ti{c}^*_{\ti{s}_i,\ti{s}_{i+1}}$.

\subsubsection{Generating partial capability descriptions}
For each candidate capability $\ti{c}^*_{\ti{s}_i,\ti{s}_{i+1}}$, 
the set of predicates ${\ti{s}_{i+1} \setminus \ti{s}_{i}}$ is
added to the effects of $\ti{c}^*_{\ti{s}_i,\ti{s}_{i+1}}$ in
positive form (add effects); whereas the set
${\ti{s}_{i} \setminus \ti{s}_{i+1}}$ is
added to the same candidate capability's effects in negative form (delete effects).
As an optimization, in a manner similar to \citet{stern_2017_efficient}, 
we also store that the predicates true in 
${\ti{s}_{i}}$ cannot be negative preconditions for this
capability, whereas the predicates false in ${\ti{s}_{i}}$ 
cannot be positive preconditions.

\subsubsection{Lifting the partial capability descriptions} 
After line 2 of \alg, we get a set of candidate capabilities with their
partial descriptions that are 
in terms of predicates $\ti{P}$ instantiated with objects in $\ti{O}$.
For each such grounded partial capability description, the predicates in the
preconditions and effects are sorted in some lexicographic order. 
The choice of ordering is not important as long as it stays consistent throughout \alg.
The objects used in predicate arguments are assigned unique IDs corresponding 
to this capability in the order 
of their appearance in ordered predicates.
These IDs are then used as variables representing
capability parameters.
E.g., suppose we have a grounded partial capability description with a precondition:
$\emph{(alive ganon), (at link cell6), (at ganon cell5),(next\_to}$ 
$\emph{ganon)}$.
Traversing the predicates in this order, 
the objects used in these predicates are given IDs as follows: 
$\{\emph{ganon: 1, link: 2, cell6: 3, cell5: 4}\}$. Note that there is only
one assignment per object, hence $\emph{ganon}$ in $\emph{(at ganon cell5)}$
was not given a separate ID. This procedure is extended to effects while
assigning new IDs 
for any unseen objects in the partial capability description.
Finally, the parameterized partial capability description 
is constructed by replacing all occurrences of objects in the partial 
capability description with variables corresponding to their unique IDs.

\subsubsection{Combining candidate capabilities} Multiple candidate
partial capabilities can be combined if their precondition and effect
conjunctions are unifiable. E.g., 
for any capability to match the capability discussed above, it's precondition
should be in the form: 
\emph{(alive ?1), (at ?2 ?3), (at ?1 ?4), (next\_to ?1)}.
Its effects should also be unifiable in terms of these parameters.
The algorithm also keeps track of which grounded partial candidate capabilities
map to each parameterized partial capability description.
These descriptions are partial as they are generated
using limited execution traces and may not capture all the preconditions
or effects of a capability. 
E.g., suppose a capability adds a literal on its execution.
If that literal is already present in the state where the capability
was executed, it will not be captured in the effect of the capability's
partial description. Hence, we next try to complete the partial
capability descriptions.
Note that all parameterized partial capability descriptions are
collectively used as the parameterized capability model $\ti{M}$ (line 4).

\subsection{Completing Partial Capability Descriptions}
\label{sec:learn}

To complete the partial capability descriptions $\ti{M}$, \alg\, 
generates queries aimed to gain more information
about the conditions under which the capability
can be executed and the state properties that become 
true or false upon its execution. These queries
give the agent a sequence of states, called waypoints, to
traverse. Based on the agent's ability to traverse
them, we derive the precondition and effect
of each capability. \alg\, iterates through the
combinations of predicates and capabilities generated earlier
to determine how each predicate will appear in each
capability's precondition and effect (line 6). 
For each combination, it generates a query
as follows.

\subsubsection{Active query generation}
For each combination of predicate, capability, and precondition (or effect),
three possible capability descriptions $M_+, M_-, M_\emptyset$ are possible, 
one each for the predicate
appearing in the precondition (or effect) of the capability in positive, negative, 
or absent mode, respectively (line 7).
As noted when generating partial capability descriptions in Sec.~3.1, 
some of the models will not be considered if we know that a form is not 
possible for a particular predicate.
The algorithm iteratively picks two such models $M_x, M_y$ from 
$M_+, M_-, M_\emptyset$ (line 8) and 
generates a query $\ti{q}$ 
in the form of a state $\ti{s}_0$ and a capability sequence $\ti{\pi}$
such that the result of executing the sequence $\ti{\pi}$ on $\ti{s}_0$ 
is different in $M_x$ and $M_y$ (line 9). 
We use the agent interrogation algorithm (AIA), from our prior work 
\citet{verma_21_asking} (henceforth referred to as VMS21).
For their process, AIA reduces query generation to a planning problem. 
The resulting query $\ti{q}$ is of the form
$\langle \ti{s}_0, \ti{\pi} \rangle$, asking the model (or an agent)
about the length of the plan $\ti{\pi}$ that it can successfully
execute when starting from state $\ti{s}_0$. Here plan $\ti{\pi}$ is a sequence
of capabilities $\langle \ti{c}^*_1, \dots, \ti{c}^*_n \rangle$ grounded with objects in $\ti{O}$.

\subsubsection{Generating waypoints from queries}
The queries described above cannot be directly posed to an agent,
as the plan $\ti{\pi}$ is in terms of high-level capabilities $\ti{c}_i^* \in \ti{C}^*$, which the agent
will not be able to comprehend. Additionally, these high-level capabilities cannot be converted directly to low-level actions,
as each capability may correspond to a different sequence of low-level actions depending on the state in which it is executed.
Hence, we pose the queries to the agent in the form of high-level state transitions induced by the capabilities in the query's capability sequence.

To accomplish this, \alg\, converts the query $\ti{q}$
to a sequence of waypoints $\ti{\varrho} = \langle \ti{s}_0,\dots,\ti{s}_n\rangle$.
Starting from the initial state $\ti{s}_0$, these are generated by applying the
capability
$\ti{c}^*_i$, for $i\in [1,n]$, in the state $\ti{s}_{i-1}$ according to the partial capability description of
$\ti{c}^*_i$.
Note that the waypoints $\ti{\varrho}$ cannot be presented to the agent
as the agent may not know the high-level vocabulary. Hence these high-level
waypoints must be refined into the low-level waypoints $\varrho = \langle s_0,\dots,s_n\rangle$
(with each $s_i$ similar to state shown in Fig.~\ref{fig:abstraction}(b)) that agent understands.

\alg\, first
converts the high-level waypoints $\ti{\varrho}$ to a sequence of
low-level waypoints $\varrho = \langle s_0,\dots,s_n\rangle$
using the predicate definitions (line 11).
Then each consecutive pair of states $\langle s_{i}, s_{i+1} \rangle$
is given sequentially to the agent as a \emph{state reachability query} asking if it can reach from
state $s_i$ to $s_{i+1}$ using its internal black-box policy (line 13). 

\subsubsection{Updating partial models based on agent responses}
Using its internal planning mechanism and the simulator, the agent attempts to reach from state
$s_i$ to $s_{i+1}$. If it succeeds, the response to the query is recorded as true;
if it fails, the response is recorded as false.
The algorithm keeps track of the waypoints that were successfully traversed. 
Based on the waypoint pairs that the agent was able to traverse,
we discard the capability descriptions
among $M_x$ and $M_y$ that are not consistent with the agent's response (line 15).

E.g., suppose the algorithm is trying to determine how the predicate $\emph{(alive ?monster1)}$
should appear in the precondition of capability $c4$ shown in Fig.~\ref{fig:abstraction}(d).
Now the two possible capability descriptions $M_1$ and $M_2$ that \alg\, is considering in
line 6 are $M_+$ and $M_-$, corresponding to 
$\emph{(alive ?monster1)}$ being in $c4$'s precondition in positive and negative form, respectively.
The algorithm will generate query with its corresponding waypoints $\ti{\varrho} = \langle \ti{s}_0,\ti{s}_1\rangle$, where 
$\ti{s}_0$ will correspond to the state shown in Fig~\ref{fig:abstraction}(a), and $\ti{s}_1$ will be 
$\ti{s}_0$ without Ganon. Now the agent uses its own internal mechanism to 
try to reach 
$\ti{s}_1$ from $\ti{s}_0$ and succeeds. Since this is not possible according to $M_-$,
$M_-$ will be discarded.

We now define and prove the theoretical properties of the capability discovery algorithm.
To do this, we use two key properties of VMS21 relevant to this work: 
(1) if there exists a distinguishing query for two models then it will be generated (Thm.~1 in VMS21); and
(2) the algorithm will not discard any model that is consistent with the agent (Thm.~2 in VMS21).
Interested readers can refer to VMS21 for further details.

\subsection{Formal Analysis}
\label{sec:guarantees}

\alg\, has two main desirable properties: 
(1) the partial capability model (that is maintained as $\ti{M}$) is 
always maximally consistent, i.e, adding any more literals into it 
would be unsupported by the execution traces that we obtain; and (2) 
the final parameterized capability is complete in the limit of infinite 
execution traces given to \alg.
We first define these concepts and then formalize the results under Thm.~1 and Thm.~2.

\begin{definition}
Let $e = \langle s_0, \dots, s_n \rangle$ be an execution trace 
with a corresponding abstract execution trace 
$\ti{e} = \langle \ti{s}_0, \dots, \ti{s}_m \rangle$, where $m \leq n$.
A \emph{parameterized capability model} $\ti{M}=\langle \ti{P}, \ti{C}, \ti{O} \rangle$ is
consistent with $E$ iff 
$\forall i \in [0,m-1]$
$\exists \ti{c}^* \in \ti{C}^* \,\,\ti{s}_{i+1} = \ti{c}^*(\ti{s}_i)$, where $\ti{C}^*$ is a set of
grounded capabilities that can be generated by instantiating the parameters of
capabilities $\ti{c} \in \ti{C}$ with objects in $\ti{O}$.
\end{definition}

We extend this terminology to say that a capability model is
consistent with a set of execution traces $E$ iff it is consistent
with every trace in $E$. This notion of consistency captures completeness
as a parameterized capability model $\ti{M}$ that is
consistent with a set of execution traces $E$, is also complete w.r.t. $E$.
We next define a stronger notion of completeness that our algorithm provides
in the form of maximal consistency. 
This helps to assess the succinctness of a capability model with a set of execution 
traces $E$.

\begin{definition}
Let $E$ be a set of execution traces, and $\Lambda$ be the set of possible agents
that can generate all execution traces in $E$. 
A \emph{parameterized capability model $\ti{M}=\langle \tilde{P}, \tilde{C}, \tilde{O} \rangle$ is
maximally consistent with a set of execution traces $E$}
iff (i) $\ti{M}$ is consistent with $E$, and (ii) adding
any predicate as positive or negative precondition or effect of a capability 
in $\ti{M}$ makes it inconsistent with at least one execution trace
that can be generated by at least one agent $\mathcal{A}^\# \in \Lambda$.
\end{definition}
 
An abstraction satisfies \emph{local connectivity} iff 
$\forall \ti{s} \,\,\forall s_i, s_j \in f^{-1}(\ti{s})$ there exists 
a sequence of
primitive actions $\langle a_i,\dots,a_n \rangle$ such that
$a_n(a_{n-1}\ldots(a_1(s_i))\ldots) = s_j$.
We use this 
to show that the parameterized capability model learned by 
\alg\, is maximally consistent.

\begin{theorem}
    Let $\mathcal{A}= \langle S, A, T \rangle$ be an agent operating in a deterministic,
    fully observable, and stationary environment with a state space $S$ using a set of
    primitive actions $A$. Given an input vocabulary $\ti{P}$, 
    and the set of execution traces $E$ generated by $\mathcal{A}$, if
    local connectivity holds, then the capability model $\ti{M}$ 
    maintained by Alg. 1 is \emph{consistent} 
    with the set of execution traces $E$.
\end{theorem}
\begin{proof}
    We show
    that given the set of all execution traces $E$, the parameterized capability model 
    $\ti{M}$ maintained by \alg\, is consistent with $E$, i.e., for every high-level transition 
    $\ti{s},\ti{s}'$ corresponding to 
    a transition in $E$, there exists a capability $\ti{c}$ which has a grounding $\ti{c}^*$ such that
    $\ti{c}^*(s) = \ti{s}'$. 
    We prove this by contradiction. The partial capability model $\ti{M}$ is initially generated using observed transitions 
    $\ti{s},\ti{s}'$ corresponding to the transitions in $E$ as grounded capabilities $\ti{c}^*_{\ti{s},\ti{s}'}$
    (lines 2 to 4 in \alg). So the model $\ti{M}$ is
    consistent with the set of traces to start with. At each step, \alg\, adds a new literal $l$ to a capability $\ti{c}$ in 
    $\ti{M}$ such that adding $l$ keeps $\ti{M}$ consistent with the agent $\mathcal{A}$ (Thm. 2 from VMS21).
    Now consider that adding $l$ to $\ti{M}$ makes it inconsistent with an execution trace in $E$, i.e., there must exist a transition $\ti{s}_1,\ti{s}_2$ such that no capability $\ti{c}^* \in \ti{C}^*$ corresponds to it.
    
    Consider the version of $\ti{c}_1$ corresponding to $c^*_{\ti{s}_1, \ti{s}_2}$ that was modified by \alg. We show that modifications inconsistent with this transition are not possible under the assumption that the agent’s capabilities can be expressed using the input vocabulary.\\
    \textit{Case 1}: Suppose \alg\, added a literal $l$ in the precondition of $\ti{c}_1$ that was not true in $\ti{s}_1$. Thm.~2 in VMS21 implies that absent and negated forms of $l$ were inconsistent with executions of $\ti{c}_1$ using the same agent that generated $E$. In other words, the agent sometimes requires $l$ as a precondition to execute $\ti{c}_1$, even though $l$ was not a part of $\ti{s}_1$. This contradicts the assumption that $\ti{c}_1$ is expressible using the input vocabulary in the form of Def.~3.\\
    \textit{Case 2}: Suppose \alg\, added a literal $l$ in the effect of $\ti{c}_1$ that was not present in $\ti{s}_2$. This implies that the negation and absence of $l$ in the result of $\ti{c}_1$ were inconsistent with the agent's execution of $\ti{c}_1$ in query-responses generated by \alg. 
    A similar contradiction about the assumption of expressiveness follows. 

    Hence, the capability model $\ti{M}$ 
    maintained by Alg.~1 is \emph{consistent} 
    with the set of execution traces $E$.
\end{proof}

\begin{theorem}
    Let $\mathcal{A}= \langle S, A, T \rangle$ be an agent operating in a deterministic,
    fully observable, and stationary environment with a state-space $S$ using a set of
    primitive actions $A$. Given an input vocabulary $\ti{P}$, 
    and the set of execution traces $E$ generated by $\mathcal{A}$, if
    local connectivity holds, then the capability model $\ti{M}$ 
    returned by Alg.~1 is \emph{maximally consistent} 
    with the set of execution traces $E$.
\end{theorem}

\begin{proof}
    We will prove the two conditions for maximal consistency separately. The first condition is that the model $\ti{M}$
    returned by \alg\, is consistent with $E$ follows directly from Thm.~1. Since the model maintained by \alg\, at each step is consistent with $E$, hence the same model returned after the last iteration is also consistent with $E$.
    
    Next, we show that
    adding any predicate as a positive or negative precondition or effect 
    of a capability in $\ti{M}$ returned by \alg\, makes it inconsistent with at least one execution trace 
    that can be generated by at least one agent $\mathcal{A}^\# \in \Lambda$, where 
    $\Lambda$ is the set of possible agents that can generate all execution traces in $E$. We prove this by contradiction. Note that a literal is not added by \alg\, to an action's precondition (or effect)
    only if (1) in the observed traces, it was not present in
    the state where (immediately after) that action was executed; or (2) adding it in the precondition (or effect) of an action
    resulted in a response to a query that was inconsistent with that of the agent.
    Also, note that a predicate corresponding to a literal is always added to the model in some form in each precondition (or effect).
    Suppose a literal $l$ that was not added by \alg\, is added
    to $\ti{M}$ in precondition (or effect) of a capability $\ti{c}$
    without making it inconsistent with the agent. 
    Since a predicate $p$ corresponding to this literal $l$ is already present in 
    $\ti{c}$, this implies that 
    the form of the predicate $p$ added by \alg\, is incorrect. But this is not possible
    as shown by Thm.~1 and Thm.~2 of VMS21. Hence this is not possible and adding an additional
    literal in any form to an action's precondition or effect would make it inconsistent with the agent.
    This means that it also makes the model inconsistent with at least one
    agent $\mathcal{A}^\# \in \Lambda$.
\end{proof}  

Next, we formalize the notion of downward refinability, that the
discovered capabilities are indeed within the
agent's scope. In this work, refinability is similar to the 
notion of forall-exists abstractions~\cite{Srivastava_16_metaphysics} 
for deterministic systems. 
Recall the notion of abstraction functions (Def.\,\ref{def:abstraction}).

\begin{definition}
Let $\ti{M}=\tuple{\ti{P}, \ti{C}, \ti{O}}$ be a capability model
with $\ti{S}$, the induced state space over $\ti{P}, \ti{O}$ using an abstraction function $f$,
for an
agent $\mathcal{A}=\tuple{S, A, T}$; and $\ti{C}^*$ be a set of
grounded capabilities that can be generated by instantiating the arguments of
capabilities $\ti{c} \in \ti{C}$ with objects in $\ti{O}$. 
A capability $\ti{c}^* \in \ti{C}^*$ is \emph{realizable w.r.t. $\mathcal{A}$} 
iff 
$\forall \ti{s}\in \ti{S}$, if
$\tilde{s}\models pre(\ti{c}^*)$ then $\forall s \in f^{-1}(\tilde{s})\quad  \exists
a_1, \ldots, a_n \in A: a_n(a_{n-1}\ldots(a_1(s))\ldots)\in \tilde{c}^*(\tilde{s})$.
The model $M$ is \emph{realizable w.r.t. $\mathcal{A}$} iff all capabilities $\ti{c}^*\in \ti{C}^*$ are realizable.
\end{definition}

In these terms, discovered capabilities are more likely to be useful
if they are accurate in the sense that they are consistent with
execution traces and realizable, i.e., true representations of what the
agent can do. Realizability captures the soundness of the model wrt the execution
of the capabilities.
We now show that the parameterized capability model that we learn is realizable.

\begin{theorem}
    Let $\ti{P}$ be a set of predicates $\ti{P}$, $\mathcal{A} = \langle S, A, T \rangle$ be an agent with a deterministic transition system $T$. If a high-level model is expressible deterministically using the predicates $\ti{P}$, 
    and
    local connectivity is ensured,
     then the parameterized capability model $\ti{M}$ learned by \alg\,
    is \emph{realizable}. 
\end{theorem}

\begin{proof}
    We will prove that for all capabilities in $\ti{C}$ learned as part
    of the parameterized capability model $\ti{M}$, for all groundings $\ti{C}^*$,
    if the capability is executed in an abstract state $\ti{s}$ such that $\tilde{s}\models pre(\ti{c}^*)$ then there exists a sequence of low-level states that the agent can traverse to reach a state $\ti{s}' \in \ti{c}^*(\ti{s})$.
    
    We prove this by cases.
    Consider a capability $\ti{c} \in \ti{C}$ whose description is learned using 
    \alg\,. Using Thm.~1, the precondition and effect of $\ti{c}$ will be consistent 
    with $E$ generated by the agent. Now consider a grounded capability $\ti{c}^*$ 
    corresponding to the capability $\ti{c}$.
    There are only two cases possible: (1) either $\ti{c}^*$ appeared in the observed
    traces or was executed successfully by the agent in response to one of the queries posed to the agent; or (2) it was not present in either. We prove each case separately.\\
    \textit{Case 1:} There exists a set of low-level states $s$ and $s'$ such that
    $\ti{c}^*(\ti{s})=\ti{s}'$, where $\ti{s}=f(s)$ and $\ti{s}'=f(s')$. Now due to local connectivity, all states in $f^{-1}(s)$ are connected with each other and same is true for all states in $f^{-1}(s')$. Hence the agent can traverse from any state in $f^{-1}(s)$ to any state in $f^{-1}(s')$ on executing the capability $\ti{c}^*$. This makes the capability $\ti{c}^*$ realizable. \\
    \textit{Case 2:}
    Since $\ti{c}^*$ was not observed directly and the only way capabilities are added to $\ti{M}$ is if they are lifted forms of capabilities identified from observation traces $E$, $\ti{c}^*$ must be a grounding of the lifted form $\ti{c}_1$ of a capability $\ti{c}_1^*$ that is of the type considered in case 1. \alg\, constructs precondition and effect of $\ti{c}_1$ while ensuring consistency with query responses and observations under the assumption that the capability model is expressible as in Def.~3. When this assumption holds, the effect or precondition of a capability can only depend on the vocabulary of available predicates, which are considered exhaustively (hierarchically) by \alg. This implies that there must be a path from a concrete state $s$ in the grounding corresponding to $\ti{c}_1$\!'s precondition to a concrete state $s'$ that satisfies the effects of grounding of $\ti{c}_1$\!'s effects. By local connectivity, this extends to all concrete states in the same abstract state as $\ti{s}$ corresponding to~$s$. 
    
    Hence  if a high-level model is expressible deterministically using the predicates $\ti{P}$, 
    and
    local connectivity is ensured,
     then the parameterized capability model $\ti{M}$ learned by \alg\,
    is \emph{realizable}. 
\end{proof}

Note that here expressibility of a high-level model refers to the class of models 
of the form defined in Def.~3.
Together, the notions of maximal consistency and realizability establish the completeness and soundness of our approach wrt a set of execution traces $E$. Note that this approach will also work when we have access to a stream of execution traces $E$ being collected at random, independent of our active querying mechanism. 
We next show that 
in the limit of infinite randomly generated execution traces, our approach will capture all possible agent capabilities with probability 1.
Here,
capturing all possible agent capabilities in a learned model 
$\ti{M} = \langle \ti{P}, \ti{C}, \ti{O} \rangle$ 
means that if
the agent can go from $\ti{s}_i$ to $\ti{s}_j$, then one of the capabilities 
in $\ti{C}$ will be instantiable to result in 
$\ti{s}_j$ when executed from $\ti{s}_i$.

\begin{theorem}
    Let $\ti{P}$ be a set of predicates, $\mathcal{A} = \langle S, A, T \rangle$ be an agent with a deterministic transition system $T$.
    Suppose random samples of agent behavior in the form of execution traces $E$ are coming from a distribution that assigns non-zero probability to at least one transition corresponding to each ground capability ($\ti{c}^*_{\ti{s}_i,\ti{s}_j}$, $\ti{s}_i,\ti{s}_j \subseteq \ti{P}$). If a high-level model is expressible deterministically using the predicates $\ti{P}$ and
    local connectivity holds, then in the limit of infinite execution traces $E$, the probability of discovering all capabilities $\ti{c}\in \ti{C}$ expressible using the predicates $\ti{P}$ is 1.
\end{theorem}
\begin{proof}
    Consider every possible abstract transition that the agent can make.
    There are finite (let's consider $L$) such transitions possible given the predicate vocabulary $\ti{P}$ and a fixed set of objects $\ti{O}$. 
    Now we are getting random  
    execution traces $E$ from a distribution that assigns non-zero probability to at least one transition corresponding to each ground capability ($\ti{c}^*_{\ti{s}_i, \ti{s}_j}$).
    This means that the probability of not observing this finite set of cardinality $L$ will reduce with each successive collection of $L$ execution traces. Hence we will eventually observe at least one transition corresponding to each ground capability ($\ti{c}^*_{\ti{s}_i,\ti{s}_j}$).
    Then as shown in Thm.~1, we will discover the capability $\ti{c}$ corresponding to the ground transition $\ti{c}^*_{\ti{s}_i,\ti{s}_j}$ with probability 1.
\end{proof}

\section{Empirical Evaluation}
\label{sec:experiments}
    
We implemented \alg\, in Python to empirically verify its 
effectiveness.~\!\!\footnote{Code: https://github.com/AAIR-lab/capability-discovery}
To show that our approach can work with different internal agent 
implementations, we evaluated Alg.\,\ref{alg:ia} with two 
broad categories of input test agents:
\emph{Policy agents} can use (possibly learned) 
black-box policies to plan and to respond to 
state reachability queries.
We used policy agents with hand-coded policies for this evaluation.
\emph{Search agents} respond to the state reachability queries using arbitrary 
search algorithms. We used 
search agents
that use A$^*$ search~\cite{hart1968formal}.
We now describe the setup of our experiments used for evaluation.
\begin{figure}[t]
  \centering
  \begin{subfigure}[t]{.48\columnwidth}
    \centering
    \includegraphics[width=\linewidth]{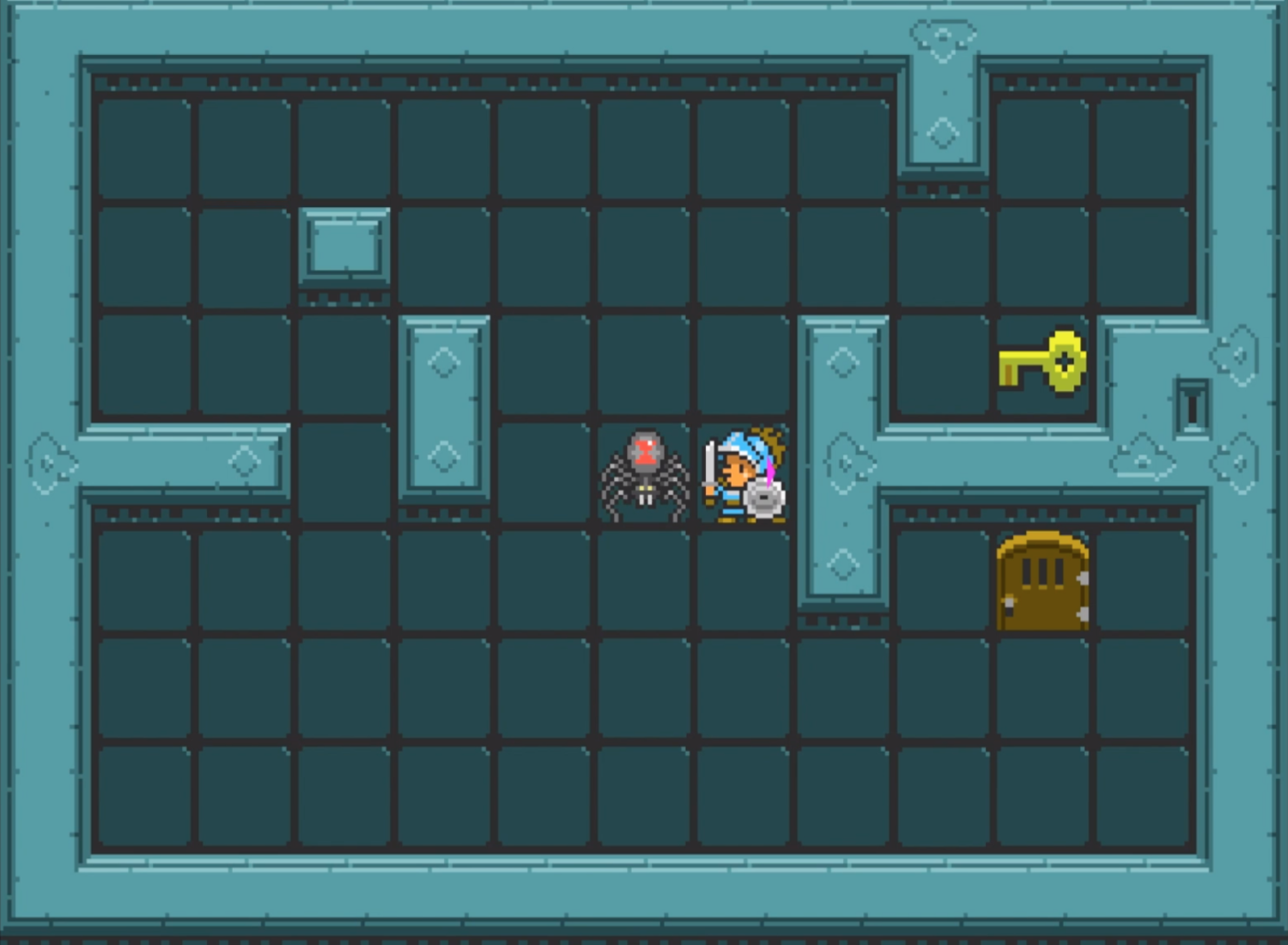}
    \caption{}
    \label{fig:gvgai_zelda}
  \end{subfigure}
  \,
  \begin{subfigure}[t]{.48\columnwidth}
    \centering
    \includegraphics[width=\linewidth]{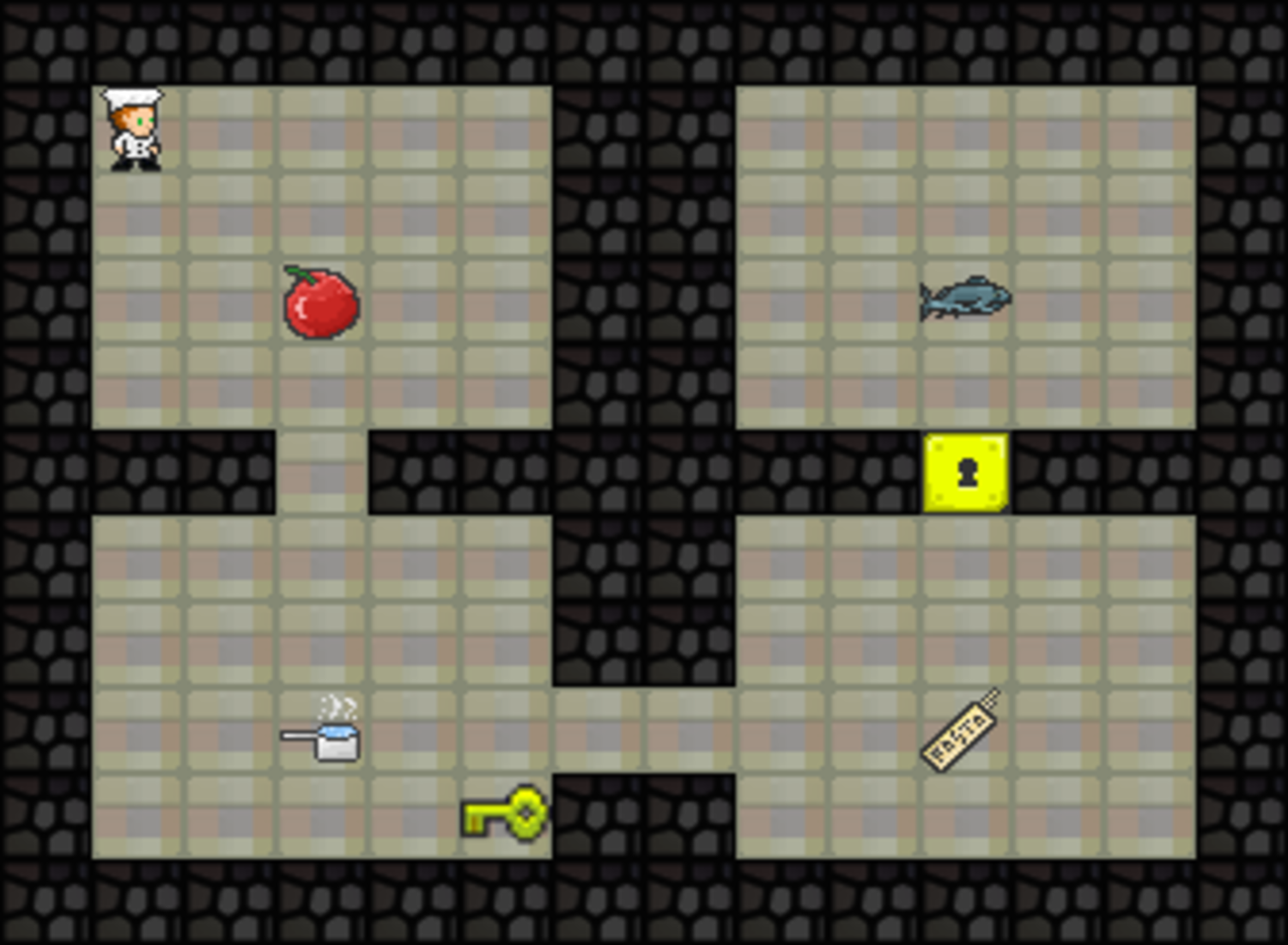}
    \caption{}
    \label{fig:gvgai_pasta}
  \end{subfigure}\\
  \,\\ 
  \begin{subfigure}[t]{.48\columnwidth}
    \centering
    \includegraphics[width=\linewidth]{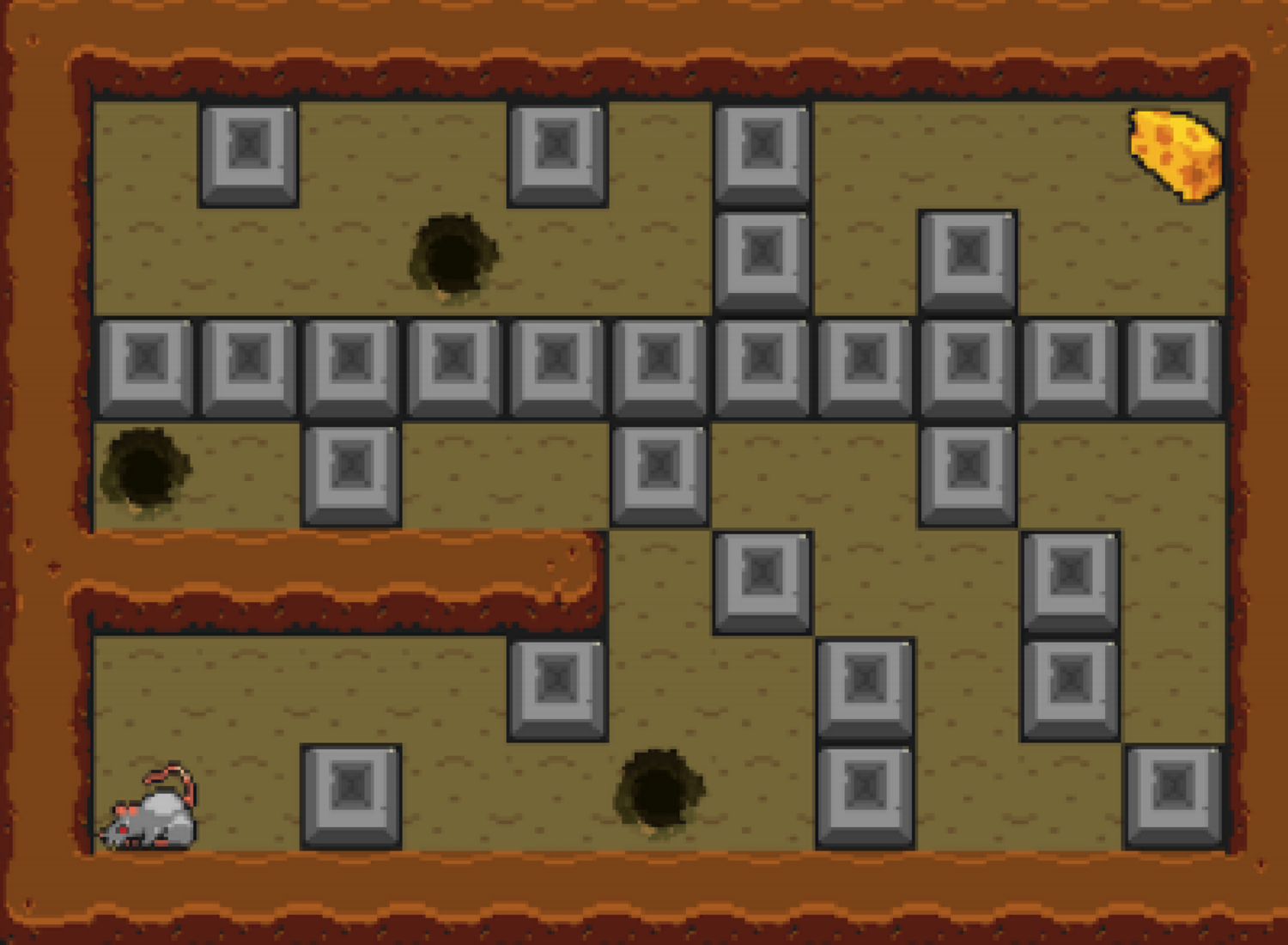}
    \caption{}
    \label{fig:gvgai_escape}
  \end{subfigure}
  \,
  \begin{subfigure}[t]{.48\columnwidth}
    \centering
    \includegraphics[width=\linewidth]{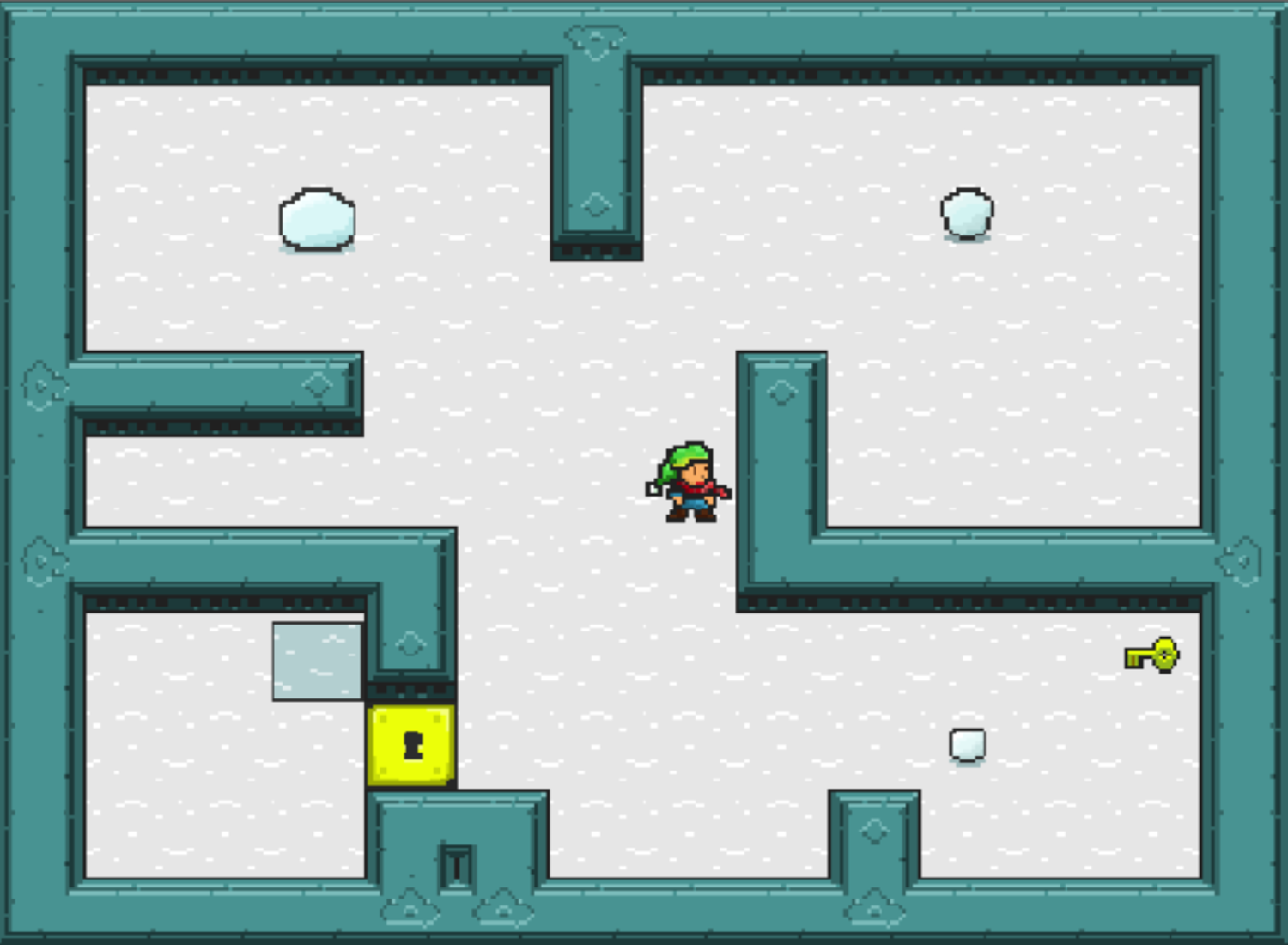}
    \caption{}
    \label{fig:gvgai_snowman}
  \end{subfigure}
  \caption{GVGAI's domains; (a) Zelda, (b) Cook-Me-Pasta, (c) Escape, and (d) Snowman.}
  \label{fig:gvgai}
\end{figure}

\subsection{Experimental Setup}
\label{sec:experiment_setup}
    
Our test agents use 
the General Video Game Artificial
Intelligence framework~\cite{Perez-Liebana_16_gvgai,Perez-Liebana_19_gvgai}.
Domains in GVGAI are 
two-dimensional ATARI-like games defined using
the Video Game Description Language PyVGDL~\cite{schaul_2013_video}.
We performed experiments on four such game domains --
Zelda,
Cook-Me-Pasta, Escape, and Snowman (Fig.~\ref{fig:gvgai}). All these domains require the agent to
navigate in a grid-based environment and complete a set of tasks (in some partial order) to complete the game. More details about these domains and the input user 
vocabularies are available in Appendix A.
Since the complete list of an agent's capabilities may be irrelevant to a 
user's current needs,    
w.l.o.g, our implementation supports an input including sets 
of formulas representing the properties that may be of interest to the user. This set can be the set of all grounded predicates in the user's concept vocabulary. 
We also consider object types to be a subset of the unary predicates in the vocabulary and assume that each object has exactly one type. These types are used and discovered in capability like any other predicate. In addition, they are used in creating parameterized capability parameters as shown in Fig.~\ref{fig:abstraction}(d).

For each domain, and for each grid size in that domain, we create a random game instance with the goal of achieving one of the user's specified properties of interest.
To generate these instances, the number of obstacles in all domains, except Escape, is set to 20\% of the total cells in the grid, whereas all other objects are generated randomly. 
We use the solution to that instance to generate the execution trace that is used in lines 1-2 of Alg.\,\ref{alg:ia}. These solutions are not always optimal. 
All experiments are run on 5.0 GHz Intel i9 CPUs with 64 GB RAM running Ubuntu 18.04.

As shown in Sec.\!~\ref{sec:guarantees}, \alg\, is guaranteed to
compute capability descriptions that are correct in the sense that they
are consistent with the execution traces, and refinable and executable
with respect to the true capabilities of the agent.
We now present the main conclusions of our empirical analysis.

We evaluated our algorithm's performance along two aspects; (i) how the performance
of our approach changes with respect to the size of the problem; and (ii) how its performance
differs for search-based vs policy-based agents.

    \begin{figure}[t]
      \centering
        \includegraphics[width=\linewidth]{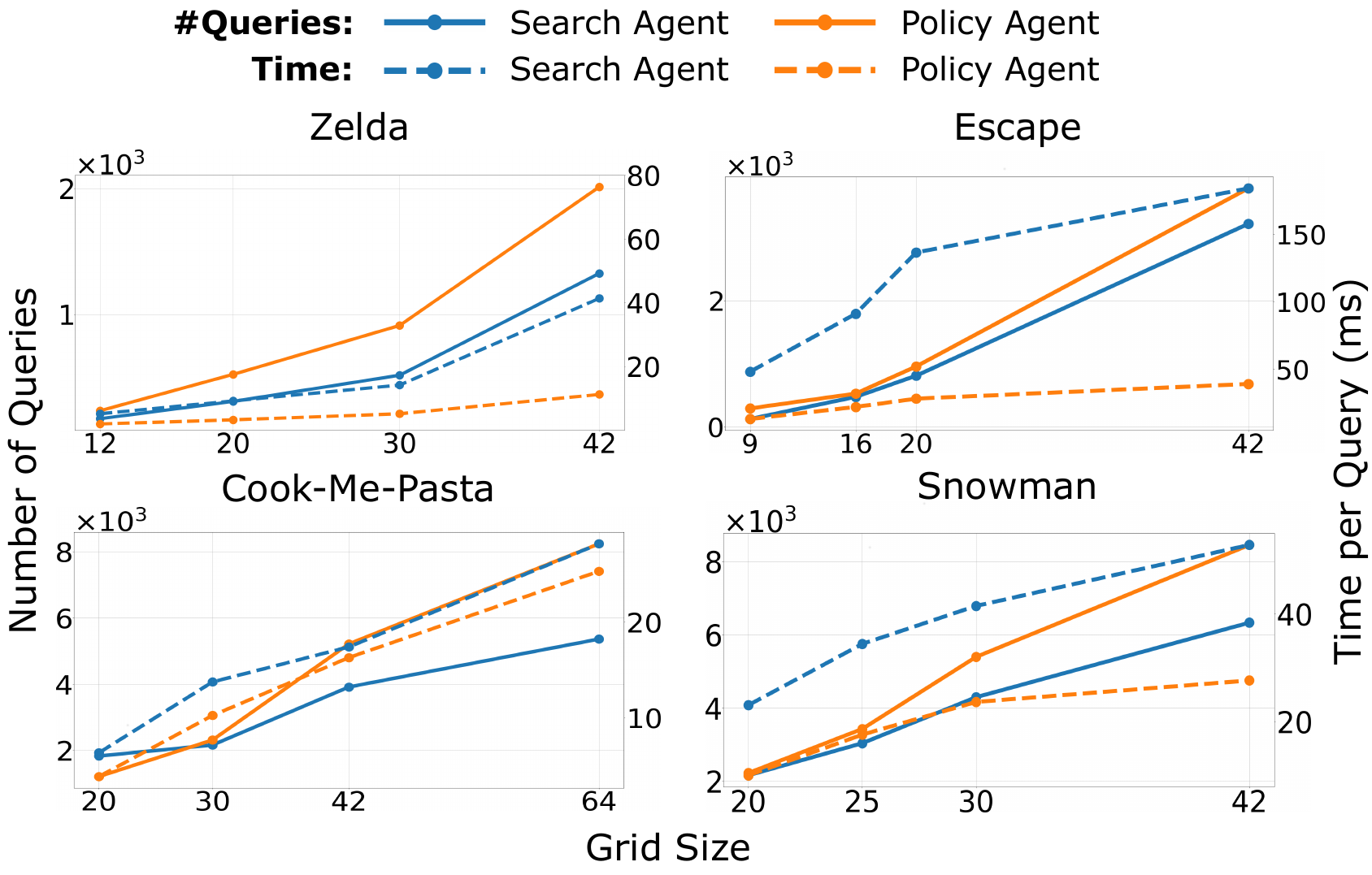}
        \caption{Performance comparison of search-based agents and policy-based agents in terms of the number of queries asked and time taken per query when increasing the grid size (number of cells in the grid) in the four GVGAI domains.}
        \label{fig:results}
    \end{figure}

\subsection{Empirical Results}
\label{sec:results}

\subsubsection{Scalability analysis} 
We increase the size of each domain to analyze its effect on
the performance of the search and policy agents. Fig.\!~\ref{fig:results} shows the graphs for the experimental runs on the four domains. In all four domains, for both kinds of agents, \emph{the number of queries increases as we increase the grid size}. 
The increasing number of queries is an expected behavior 
and this is also clear in approaches that use passive observations of agent behavior~\cite{Yang2007,aineto2019learning}.  

\subsubsection{Agent type analysis} 
The \emph{number of queries required by the policy agent is higher than that of the search agent} in almost all cases. This is because a large number of state reachability queries fail on the policy agent as the sequence of waypoints in these queries does not always align with the policy of the agent. 
However, \emph{the time per query is lesser for the policy agents} as they can answer the 
state reachability 
queries by following their policy, whereas the search agents perform an exhaustive search of the state space for every such query.

\subsection{User Study}
\label{sec:user_study}

We conducted a user study to evaluate 
the utility of  the  capability descriptions discovered and computed by Alg.\,\ref{alg:ia}.
Intuitively, our notion of interpretability matches that of common English and its use in AI literature, e.g., as enunciated by
\citet{Doshi-Velez2018}: \emph{``the ability to explain or to present in understandable terms to a human.''} We evaluate
  this through the following operational hypothesis:

\noindent
\textbf{H1.}
The discovered capabilities 
make it easier for users to 
analyze and predict the outcome of the agent’s possible behaviors.

We designed the following study to evaluate H1. 

\begin{figure}
    \centering
    \includegraphics[width=\columnwidth]{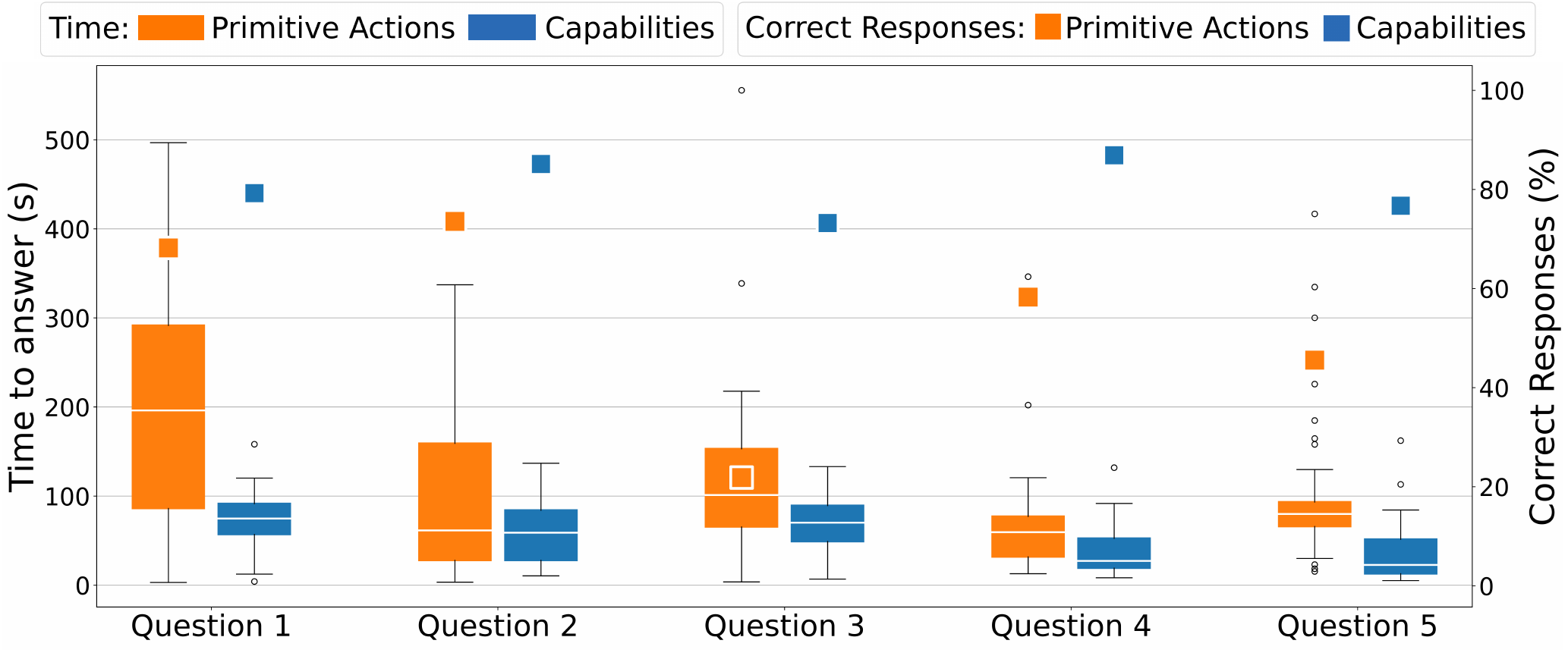}
    \caption{ Data from behavior analysis shows that using computed capability descriptions took lesser time and yielded more accurate results. See Sec.\!~\ref{sec:user_study} for details.}
    \label{fig:behavior}
\end{figure}

\subsubsection{Behavior analysis study} This study compares the predictability and analyzability of
agent behavior 
in terms of the agent's low-level actions and high-level capabilities. Each user is explained the rules of a Zelda-like game.
One group of users -- called the \textit{primitive action group} -- are presented with text descriptions of
the agent's primitive actions, while the users in the other group -- called the \textit{capability group} -- are 
presented with a text description of the six learned capabilities. 
The capability group users are
asked to choose a short summarization for each capability description, out of the eight possible
summarizations that we provide, whereas the primitive action group users are asked to choose a short summarization for each primitive action description, out of the five possible
summarizations that we provide.
Then each user is given the same 5 questions in order. Each question contains two game-state images; start and end state. The user is asked what sequence of actions or capabilities
that
the agent should execute to reach the end state from the start state. Each question has 5 possible options
for the user to choose from, and these options differ depending on their group. We then collect the data about the accuracy of the answers, and the time
taken to answer each question.

\subsubsection{Study design} 108 participants
were recruited from 
Amazon Mechanical Turk
and randomly divided into two groups 
of 54 each.
Each user was provided with a survey on 
Qualtrics~\cite{qualtrics_2022}
that explained the rules of GVGAI's Zelda game.
We used
screeners~\cite{kennedy_2020_shape,arndt2021collecting} to ensure quality of the data collected, and discarded 23 responses. The results are based on the responses of 41 and 43 users in
the primitive action and capability group, respectively.

\subsubsection{Results} 
The results of the behavior analysis study  are shown in (Fig.\!~\ref{fig:behavior}) 
To evaluate the statistical significance\linebreak
(p-value) of the difference in the mean of the
time taken by the two groups, we used Student's t-test~\cite{student1908probable}.
The results
indicate that the test results were statistically significant 
with p-values less than 0.05 for all five questions.
Also,
the users took less time to answer questions and they got more responses correct
when using the capabilities as compared to using primitive actions.
This validates H1 that the discovered capabilities made it easier for the users to analyze and predict the agent's behavior correctly.
Detailed  information about the user study 
is available in Appendix B.

{
\section{Related Work}
\label{sec:related_work}
}

\subsubsection{High-level skills from input options}
Given a set of options encoding skills as input, \citet{konidaris_18_from}
and \citet{james_20_learning} propose methods for learning 
high-level propositional models of options representing various ``skills.'' 
They
assume access to predefined options and learn the high-level symbols
that describe those options at the high-level. 
While they use options or skills as inputs to learn models defining when those skills will be useful in terms of auto-generated symbols (for which explanatory semantics could be derived in a post-hoc fashion), our approach uses user-provided interpretable concepts as apriori inputs to learn agent capabilities: high-level actions as well as their interpretable descriptions in terms of the input vocabulary.

\subsubsection{Learning symbolic models using physics simulators}
Multiple approaches learn different kinds of symbolic models of the functionality of
ATARI or physics-based simulators using methods like conjunctions of binary input
features~\cite{kansky_17_schema}, graph neural networks~\cite{battaglia_16_interaction,crarmer_20_discovering},
CNNs~\cite{agrawal_16_learning,fragkiadaki_16_learning}, etc.
Some methods create interpretable
descriptions of reinforcement learning policies using trees~\cite{liu_2018_toward}
or specialized programming languages~\cite{verma_2018_programmatically}.
These approaches solve the orthogonal problem of learning the functionality of an agent 
that could help a user understand how an agent would solve a problem, whereas we focus
on learning capabilities of the agent that could help a user understand and answer what type
of problems it could solve.

\subsubsection{Action model learning} 
The planning community has also worked on learning STRIPS-like action models of
agent functionality from observations of its
behavior~\cite{gil_94_learning,Yang2007,Cresswell09,Zhuo13action,stern_2017_efficient,aineto2019learning,bonet_2020_learning}.
\citet{jiminez_2012_review}
and \citet{arora_2018_review} present a comprehensive review of such approaches.
These methods work with broad assumptions that
the agent model is internally expressed in the same vocabulary as
the user's~\cite{gil_94_learning,weber_11_goal,juba_2021_safe}, or at a similar
level of abstraction~\cite{mehta_11_autonomous,verma_21_asking,nayyar2022differential}. 
Additionally, such methods have as input a given set of predicates 
in terms of which they learn the functionality descriptions of the agent.

\subsubsection{High-level actions}
Works like \citet{Madumal_2020_explainable} explain 
an agent's policy in terms of high-level actions but 
they assume that high-level actions are a part of the input 
whereas our approach discovers these actions.
There is an orthogonal thread of research on using high-level actions
in AI planning as tasks, and learning low-level policies
for each of those 
tasks~\cite{yang2018peorl,lyu2019sdrl,illanes2020symbolic,kokel2021reprel}. 
These works assume the high-level actions as input and learn
the corresponding low-level policies.

As compared to the above two classes of methods, our work
focuses on solving the harder problem of discovering the capabilities of the agent behavior resulting
from its planning/learning algorithms
and learning the descriptions of these capabilities. 

\section{Conclusion}
\label{sec:conclusion}

We presented a novel approach for learning the capability description
of an AI system in terms of user-interpretable concepts by 
combining information from passive
execution traces and active query answering. 
Our approach works for settings where 
the user's conceptual vocabulary is imprecise and
cannot directly express the agent's capabilities.
Our empirical analysis showed that for the agents that internally use black-box
deterministic policies, or search techniques, we can
successfully discover the capabilities and their descriptions.
Extending this approach for partially observable settings and 
relaxing the various assumptions we made are some of the promising future directions
for this work.

\section*{Acknowledgements} \label{sec:ack}
We thank Nancy Cooke, Akkamahadevi Hanni, and Sydney Wallace
for their help with the user study.
We also thank anonymous reviewers for 
their helpful feedback on the paper.
This work was supported in part by the NSF under grants IIS 1942856,
IIS 1909370, and the ONR grant N00014-21-1-2045.

\bibliographystyle{kr}
\bibliography{capability}

\cleardoublepage
\appendix
\setcounter{theorem}{0}

\section{Domains and their Semantics}
\label{sec:domains_app}

This section describes the four GVGAI game domains used in this work and the semantics of the user interpretable predicates in these domains. 
Note that information like the orientation of the agent (player) in each of these domains is not captured
by any of the predicates. This information is important for low-level policies as certain actions can 
only be executed in certain orientations.

\subsection{Zelda}
The Zelda-like domain, as shown in Fig.\!~\ref{fig:abstraction}a, consists of a key, a door that opens using that key, the antagonist player \textit{Link}, and the protagonist monster \textit{Ganon}. To win the game, Link must defeat Ganon, and then should use the key to open the door to escape. Link can move one cell at a time in the direction it is facing. If Link moves into the cell adjacent to the key, Link picks up the key by executing the keystroke \texttt{E} (special keystroke). The same keystroke is used to Defeat Ganon when Link is facing Ganon and is in a cell adjacent to Ganon, and to escape when Link is in a cell adjacent to the door and facing it.
The user vocabulary for this domain is shown in Tab.~\ref{tab:zelda}.
\begin{table}[h]
    \centering
    \small
    \rowcolors{2}{gray!13}{}
    \begin{tabular}{p{0.28\columnwidth}p{0.62\columnwidth}}
    \toprule
    \textbf{Predicate} & \textbf{Meaning}\\
    \midrule
    
    \texttt{at(?ob ?loc)} & True if an object \texttt{?ob} is at location \texttt{?loc}\\
    \texttt{wall(?loc)} & True if there is a wall at location \texttt{?loc}\\
    \texttt{clear(?loc)} & True if location \texttt{?loc} is empty, i.e., it has no object, wall, or player\\
    \texttt{has\_key()} & True if Link has the key. \\
    \texttt{escaped()} & True if Link has escaped (game is over).\\
    \texttt{alive(?m)} & True if Ganon is still alive\\
    \texttt{next\_to(?ob2)} & True if Link is in a cell adjacent to ganon, door, or key.\\
    \bottomrule
    \end{tabular}
    \caption{Predicates in the user vocabulary for Zelda}
    \label{tab:zelda}
    \end{table}

\subsection{Cook-Me-Pasta}
The Cook-Me-Pasta domain, as shown in Fig.\!~\ref{fig:gvgai_pasta}, consists of raw pasta, sauce, boiling water, tuna (fish), lock, and key. The objective is to cook tuna pasta using a three-step process. First, the pasta is cooked by adding boiling water to the raw pasta, this can be done by pressing \texttt{E} while holding both the ingredients. Similarly, tuna is cooked by mixing sauce and tuna. Finally, the cooked pasta and the cooked tuna are to be mixed together. One or more of the ingredients can be locked in a room which must be opened using a key.
The user vocabulary for this domain is shown in Tab.~\ref{tab:pasta}.
\begin{table}[h]
    \centering
    \small
    
     \rowcolors{2}{gray!13}{}
     \begin{tabular}{p{0.28\columnwidth}p{0.62\columnwidth}}
    \toprule
    \textbf{Predicate} & \textbf{Meaning}\\
    \midrule
    
    \texttt{at(?ob ?loc)} & True if an object \texttt{?ob} is at location \texttt{?loc}\\
    \texttt{wall(?loc)} & True if there is a wall at location \texttt{?loc}\\
    \texttt{clear(?loc)} & True if location \texttt{?loc} is empty, i.e., it has no object, wall, or player\\
    \texttt{has\_key()} & True if the player has the key \\
    \texttt{pasta\_cooked()} & True if the pasta is cooked\\
    \texttt{is\_door(?loc)} & True if the location \texttt{?loc} has a door\\
    \bottomrule
    \end{tabular}
    \caption{Predicates in the user vocabulary for Cook-Me-Pasta.}
    \label{tab:pasta}
    \end{table}

\subsection{Escape}
The Escape domain, as shown in Fig.\!~\ref{fig:gvgai_escape}, consists of movable blocks, fixed holes, and cheese. The blocks can be pushed into the holes to clear out a path. The game is finished when the player reaches the location with cheese.
The user vocabulary for this domain is shown in Tab.~\ref{tab:escape}.

\begin{table}[t]
    \centering
    \small
     \rowcolors{2}{gray!13}{}
     \begin{tabular}{p{0.28\columnwidth}p{0.62\columnwidth}}
    \toprule
    \textbf{Predicate} & \textbf{Meaning}\\
    \midrule
    
    \texttt{at(?ob ?loc)} & True if an object \texttt{?ob} is at location \texttt{?loc}\\
    \texttt{wall(?loc)} & True if there is a wall at location \texttt{?loc}\\
    \texttt{clear(?loc)} & True if location \texttt{?loc} is empty, i.e., it has no object, wall, or player\\
    \texttt{is\_hole(?loc)} & True if the location \texttt{?loc} has a hole \\
    \texttt{is\_goal(?loc)} & True if the location \texttt{?loc} is the goal location \\
    \texttt{is\_block(?loc)} & True if the location \texttt{?loc} has a movable block\\
    \bottomrule
    \end{tabular}
    \caption{Predicates in the user vocabulary for Escape.}
    \label{tab:escape}
\end{table}

\subsection{Snowman}
The Snowman domain, as shown in Fig.\!~\ref{fig:gvgai_snowman}, consists of three pieces of a snowman: the top, middle, and bottom piece; a key that can be used to unlock a door (like other domains), and the goal cell. The objective of the game is to assemble the snowman in the goal location in order, constrained by the player being able to hold only one piece at any given time.
The user vocabulary for this domain is shown in Tab.~\ref{tab:snowman}.
\begin{table}[h]
    \centering
    \small
    
     \rowcolors{2}{gray!13}{}
     \begin{tabular}{p{0.28\columnwidth}p{0.62\columnwidth}}
    \toprule
    \textbf{Predicate} & \textbf{Meaning}\\
    \midrule
    
    \texttt{at(?ob ?loc)} & True if an object \texttt{?ob} is at location \texttt{?loc}\\
    \texttt{wall(?loc)} & True if there is a wall at location \texttt{?loc}\\
    \texttt{clear(?loc)} & True if location \texttt{?loc} is empty, i.e., it has no object, wall, or player\\
    \texttt{has\_key()} & True if the player has the key \\
    \texttt{player\_has(?o)} & True if the player has object \texttt{?o} \\
    \texttt{is\_goal(?loc)} & True if the location \texttt{?loc} is the goal location \\
    \texttt{placed(?part)} & True if part \texttt{?part} is placed at the goal location.\\
    \texttt{is\_door(?loc)} & True if the location \texttt{?loc} has a door\\
    \bottomrule
    \end{tabular}
    \caption{Predicates in the user vocabulary for Snowman.}
    \label{tab:snowman}
\end{table}

\section{User Study Details}
In this section, we describe the details of the user study survey that was given to the
study participants. The participants were split into two groups. The capability group and the primitive action group. 

\mysssection{Game description} The participants in both groups were shown the description of the game. As shown in Fig.~\ref{fig:rules}, this description lists out the rules of the game.

\begin{figure}[h]
    \centering
    \includegraphics[width=\columnwidth]{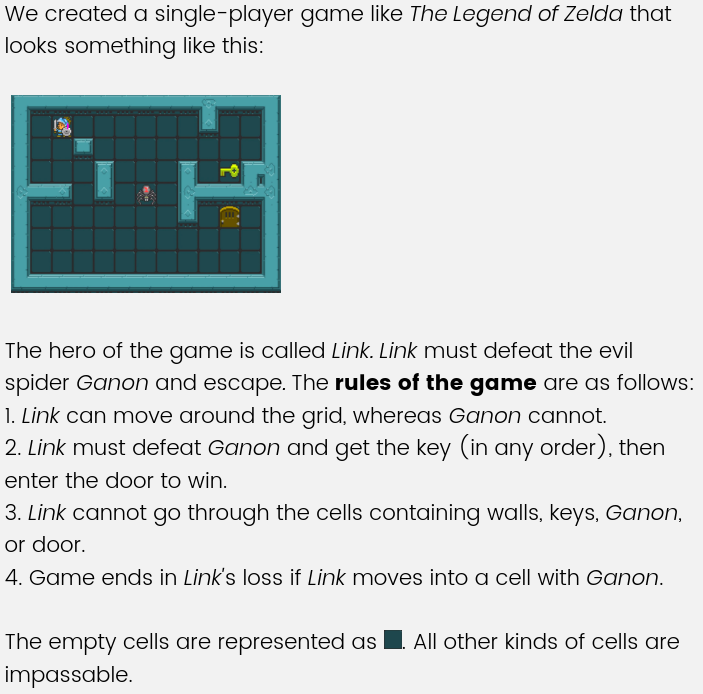}
    \caption{Game description shown to the study participants}
    \label{fig:rules}
\end{figure}

\mysssection{Capability descriptions} The participants are then shown the next part based on which group they fall in. The participants in the capability group are shown description of 6 parameterized actions,
each generated using boilerplate templates for each predicate. We show here (Fig.~\ref{fig:capabilities}) the description of the capability $c4$ whose learned description was shown in Fig.1(d) in the main paper. The participants are also given an option to choose between eight possible descriptions of from which they choose the correct summarization of that capability. This is illustrated in Fig.~\ref{fig:capabilities}. 

\begin{figure}[!h]
    \centering
    \includegraphics[width=\columnwidth]{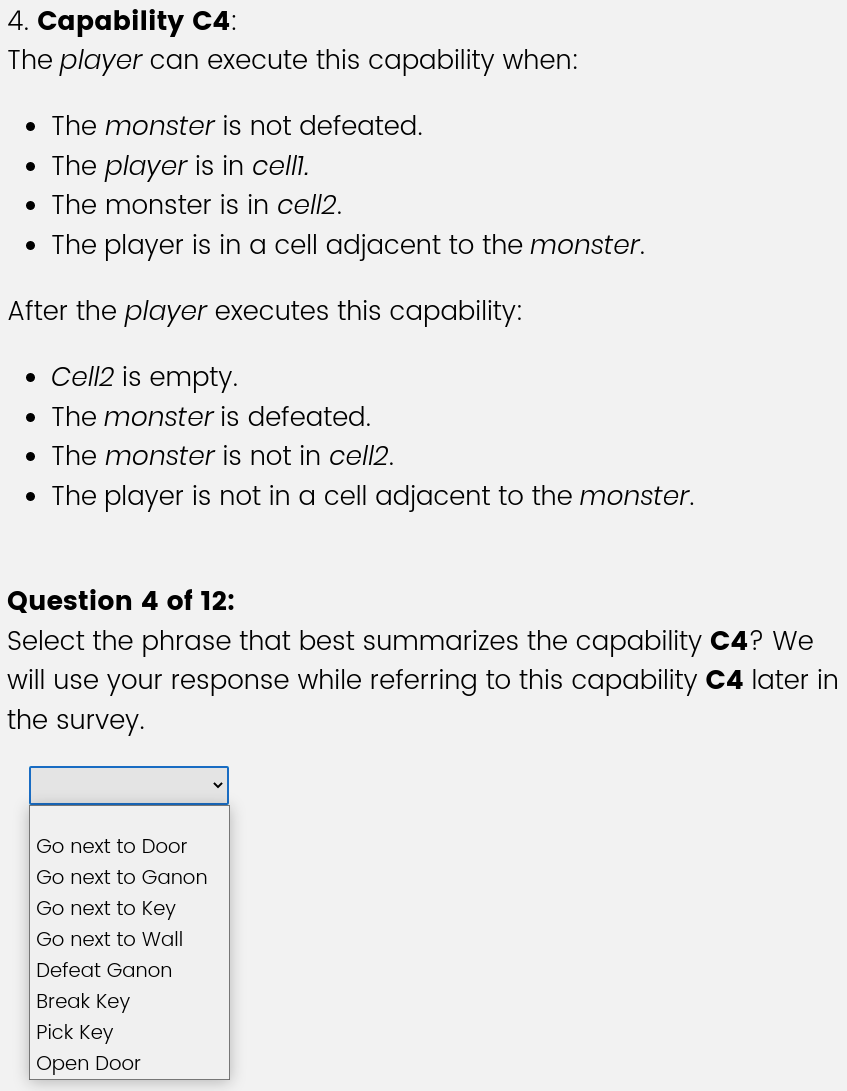}
    \caption{Description of the capability C4 with summarization options.}
    \label{fig:capabilities}
\end{figure}

\mysssection{Action descriptions} Similar to the capability group, the participants in the primitive action group are shown textual descriptions of the keystrokes, with five options to choose from. Each option provides a possible description of the keyboard in English. Fig.~\ref{fig:actions} shows the description of keystroke $W$ with the five options.

Notice that we tried to keep the same format for the description of actions as that of capabilities, i.e., of the form ``if $\langle x\rangle$ conditions hold then $\langle y\rangle$ happens.'' Also, the description of capabilities are parameterized by the player, monster, cells, etc. whereas the description in primitive actions use the object names like Link, Gannon, etc. directly. 

\begin{figure}[!h]
    \centering
    \includegraphics[width=\columnwidth]{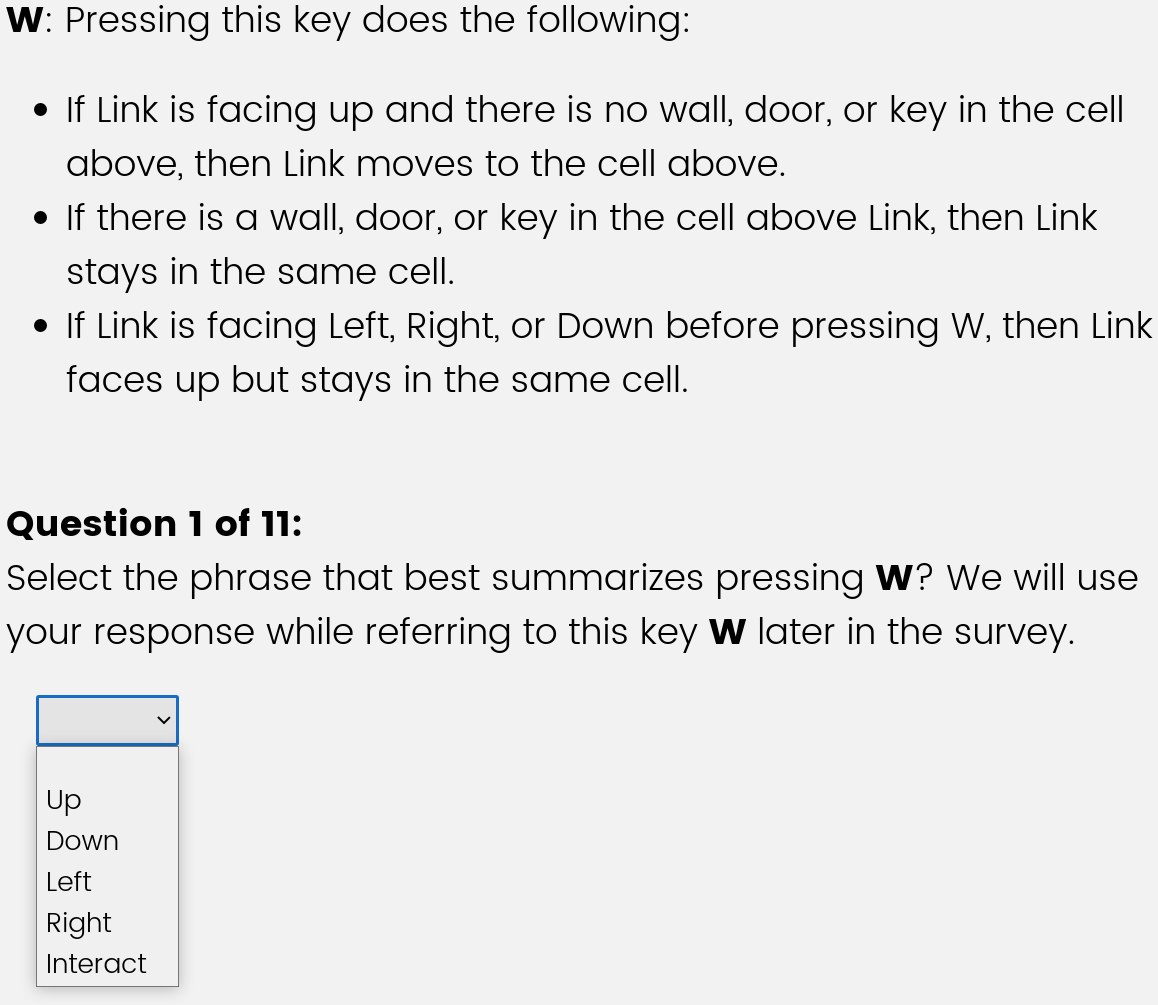}
    \caption{Description of the keystroke $W$ with summary options}
    \label{fig:actions}
\end{figure}

\mysssection{Questions} After showing the capability and action descriptions, the participants of both the groups are shown the same questions. These questions give two-game images and ask the participant the sequence of capabilities or actions (depending on the user's group) that the agent should execute to reach the goal state from the initial state. One such question is shown in Fig.~\ref{fig:sample_question}. There were six such questions in total shown.in total to all the participants. \\
\emph{Sanity Question:} One of these six was a sanity check question. The answer was given in the question itself. The responses for any participant who got this question wrong were discarded.

\begin{figure}[h]
    \centering
    \includegraphics[width=\columnwidth]{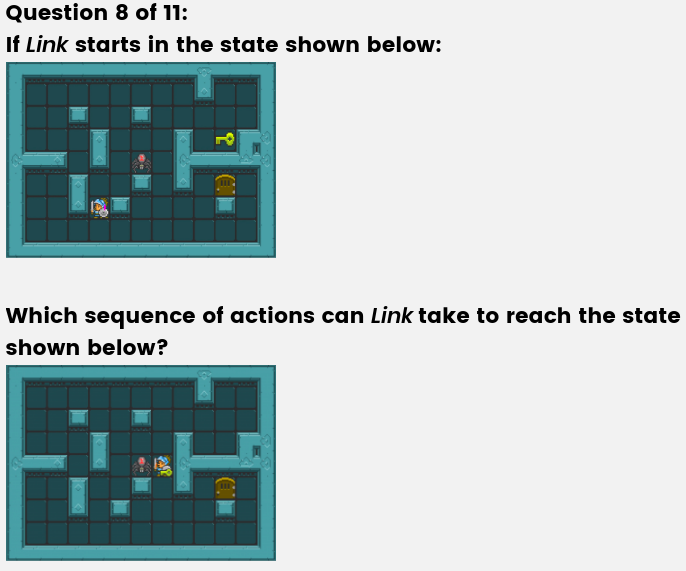}
    \caption{A sample user study question}
    \label{fig:sample_question}
\end{figure}

\mysssection{Options} The options given to the two sets of users for the same question
differed because the capability group participants were given options in terms of capability sequence that the agent can execute (shown in Fig.~\ref{fig:cap_options}), whereas the primitive action group participants were given options in terms of sequences of primitive actions (shown in Fig.~\ref{fig:act_options}). Note that these options refer to the question shown in Fig.~\ref{fig:sample_question}. 

\begin{figure}[h]
    \centering
    \includegraphics[width=\columnwidth]{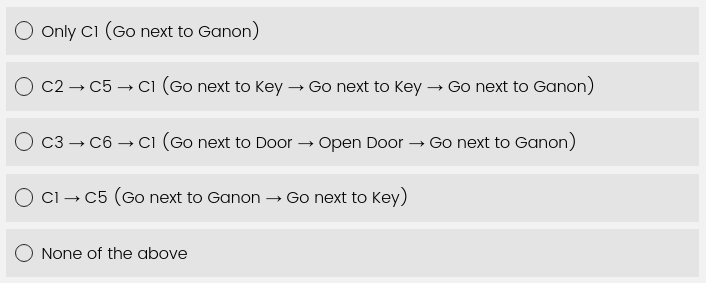}
    \caption{Options for question in Fig.~\ref{fig:sample_question} given to capability group participants}
    \label{fig:cap_options}
    
\end{figure}

\begin{figure}[!h]
    \centering
    \includegraphics[width=\columnwidth]{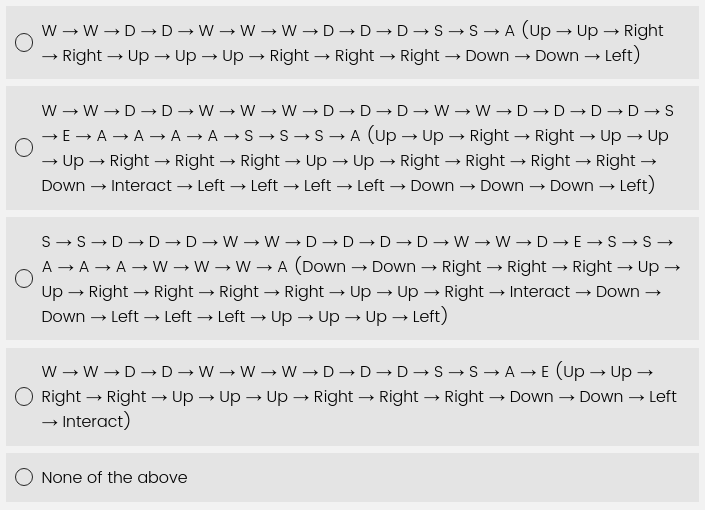}
    \caption{Options for question in Fig.~\ref{fig:sample_question} given to primitive action group participants}
    \label{fig:act_options}
\end{figure}

\section{Secondary User Study}
We also investigated another hypothesis assessing whether the users were able to 
understand the descriptions by assessing whether they can effectively summarize the capabilities. We formalize the hypothesis as:

\noindent
\textbf{H2.}
The user can effectively summarize the learned capability descriptions.

We performed the following study to evaluate the hypothesis:

\mysssection{Capability summarization study} This study evaluates the interpretability of the discovered 
capability descriptions. 
The user is explained the rules of the Zelda-like game described earlier (shown in Fig.~\ref{fig:rules}), and then  presented with a text description of the six learned capabilities. Finally, as shown in Fig.~\ref{fig:capabilities}, the user is
asked to choose a short summarization for each description, out of the eight possible
summarizations that we provide. 

\mysssection{Results} There were a total of 54 participants in the capability group out of whom 43 got the sanity check question right. The results of the capability summarization study (Tab.\!~\ref{tab:summary}) for these 43 participants
demonstrate that the users are able to summarize the descriptions almost uniformly accurately except for C3 and C5. This verifies H2 that the users can effectively summarize the learned capability descriptions.

\begin{table}
\rowcolors{2}{gray!10}{}
\centering
\footnotesize
\begin{tabular}{lcccccccc}
\toprule  
& S1 & S2 & S3 & S4 & S5 & S6 & S7 & S8 \\ 
 \hline
C1 & \cellcolor{green}1.0 & 0 & 0 & 0 & 0 & 0 & 0 & 0 \\
C2 & 0 & \cellcolor{green}1.0 & 0 & 0 & 0 & 0 & 0 & 0 \\
C3 & 0 & 0 & \cellcolor{green}0.91 & 0 & 0 & 0 & \cellcolor{orange}0.09 & 0\\
C4 & 0 & 0 & 0 & \cellcolor{green}1.0 & 0 & 0 & 0 & 0 \\
C5 & 0 & 0 & 0 & 0 & \cellcolor{green}0.84 & 0 & 0 & \cellcolor{orange}0.16\\
C6 & 0 & 0 & 0 & 0 & 0 & \cellcolor{green}1.00 & 0 & 0\\
\hline
\end{tabular}
\caption{ {Accuracy of capability summarization study for the Zelda-like game. An element in 
row Ci and column Sj represents the fraction of instances when capability Ci was summarized as 
Sj by the study participants. Correct summarization of Ci is Si (in green). C1,S1: \emph{Go next to Ganon}; C2,S2: \emph{Go next to Key}; C3,S3: 
\emph{Go next to Door}; C4,S4: \emph{Defeat Ganon}; C5,S5: \emph{Pick Key}; C6,S6: \emph{Open Door}; 
S7: \emph{Go next to Wall}; S8: \emph{Break Key}.}}
\label{tab:summary}
\end{table}

\end{document}